%% file: main_arxiv.tex
\documentclass[11pt]{article}

\input{macros_arxiv}

\title{\bfseries\papertitle}

\author{
Aadirupa Saha$^1$\\
UIC
\and
Aniket Wagde\\
UIC
\and
Branislav Kveton\\
Adobe
}

\date{}

\begin{document}

\maketitle

\input{abstract.tex}

\input{intro.tex}

\input{problem.tex}

\input{algo_known.tex}

\input{algo_unknown.tex}

\input{expts_new.tex}

\input{conclusion.tex}

\bibliographystyle{plainnat}
\bibliography{bandit_refs.bib,qllm_judge.bib,shvar.bib}

\newpage
\appendix
\onecolumn

\input{appendix}

\end{document}

%% file: macros_arxiv.tex
\usepackage[utf8]{inputenc} 
\usepackage{fullpage}
\usepackage[T1]{fontenc}    
\usepackage{url}            
\usepackage{booktabs}       
\usepackage{amsfonts}       
\usepackage{nicefrac}       
\usepackage{microtype}      
\usepackage{xcolor}         
\usepackage{microtype}
\usepackage{subfigure}
\usepackage{tcolorbox}

\usepackage{amsmath}        
\usepackage{amsthm}
\usepackage{mathrsfs}
\usepackage{bbm}            
\usepackage{mathtools}
\usepackage{xcolor}
\usepackage[extdef=true]{delimset}
\usepackage{algorithm}
\usepackage{algorithmicx}
\usepackage[noend]{algpseudocode}
\usepackage{nicefrac}
\usepackage{thm-restate}
\usepackage{enumitem}

\usepackage[textsize=tiny,textwidth=1in]{todonotes}

\newcommand{\red}[1]{\textcolor{red}{[#1]}}

\newcommand{\as}{} 

\newcommand{\papertitle}{LLM-as-Judge on a Budget} 

\usepackage[colorlinks, linkcolor=purple, filecolor = blue, citecolor = blue, urlcolor = olive]{hyperref}
\usepackage[capitalise]{cleveref}
\usepackage{natbib}

\allowdisplaybreaks

\newtheorem{thm}{Theorem}

\newtheorem{cor}[thm]{Corollary}
\newtheorem{defn}[thm]{Definition}

\newtheorem{rem}{Remark}

\newcommand{\R}{{\mathbb R}}

\newcommand{\N}{{\mathbb N}}

\newcommand{\1}{{\mathbf 1}}

\newcommand{\cA}{{\mathcal A}}

\newcommand{\cB}{{\mathcal B}}

\newcommand{\cS}{{\mathcal S}}

\newcommand{\cG}{{\mathcal G}}

\newcommand{\cN}{{\mathcal N}}

\newcommand{\cR}{{\mathcal R}}
\newcommand{\cH}{{\mathcal H}}

\newcommand{\tO}{{\tilde {O}}}

\newcommand{\hs}{\hat{s}}

\crefname{algorithm}{Alg.}{Algs.}
\Crefname{algorithm}{Alg.}{Algs.}
\crefname{thm}{Thm.}{Thms.}
\Crefname{theorem}{Theorem}{Theorems}

\crefname{lem}{Lem.}{Lems.}
\Crefname{lemma}{Lemma}{Lemmas}

\crefname{cor}{Cor.}{Cors.}
\Crefname{corollary}{Corollary}{Corollaries}

\crefname{proposition}{Prop.}{Props.}
\Crefname{proposition}{Proposition}{Propositions}

\crefname{rem}{Rem.}{Rems.}
\Crefname{remark}{Remark}{Remarks}

\crefname{defn}{Def.}{Defs.}
\Crefname{definition}{Definition}{Definitions}

\crefname{example}{Ex.}{Exs.}
\Crefname{example}{Example}{Examples}

\crefname{assumption}{Assum.}{Assums.}
\Crefname{assumption}{Assumption}{Assumptions}

\crefname{section}{Sec.}{Secs.}
\Crefname{section}{Section}{Sections}

\crefname{appendix}{App.}{Apps.}
\Crefname{appendix}{Appendix}{Appendices}

\crefname{equation}{Eq.}{Eqs.}
\Crefname{equation}{Equation}{Equations}

\crefname{figure}{Fig.}{Figs.}
\Crefname{figure}{Figure}{Figures}

\crefname{table}{Tab.}{Tabs.}
\Crefname{table}{Table}{Tables}

\DeclareMathOperator*{\argmax}{argmax}

\mathchardef\mhyphen="2D

\newcommand{\pr}[1]{\mathbb{P} \left(#1\right)}

\newcommand{\var}[1]{\mathrm{var} \left[#1\right]}

\newcommand{\bign}[1]{\big(#1\big)}
\newcommand{\biggn}[1]{\bigg(#1\bigg)}

\renewcommand{\norm}[1]{\|#1\|}

\newcommand{\rnd}[1]{{#1}}
\renewcommand{\set}[1]{\left\{#1\right\}}

\newcommand{\myfloor}[1]{\big  \lfloor {#1} \big \rfloor}
\newcommand{\myceil}[1]{\big \lceil {#1} \big \rceil}
\newcommand{\ucb}[2]{{\bar V}_{#1}(#2)}

\def \algk {\texttt{ROBIN}}
\def \algu {\texttt{ROBIN-HOOD}}


%% file: abstract.tex
\begin{abstract}
%
LLM-as-a-judge has emerged as a cornerstone technique for evaluating large language models by leveraging LLM reasoning to score prompt-response pairs. Since LLM judgments are stochastic, practitioners commonly query each pair multiple times to estimate mean scores accurately. This raises a critical challenge: given a fixed computational budget $B$, how to optimally allocate queries across $K$ prompt-response pairs to minimize estimation error? 
We present a principled variance-adaptive approach leveraging multi-armed bandit theory and concentration inequalities. Our method dynamically allocates queries based on estimated score variances, concentrating resources where uncertainty is highest. 
Further, our algorithm is shown to achieve a worst-case score-estimation error of  $\tilde{O}\left(\sqrt{\frac{\sum_{i=1}^K \sigma_i^2}{B}}\right)$, $\sigma_i^2$ being the unknown score variance for pair $i \in [K]$ with near-optimal budget allocation. 
Experiments on \emph{Summarize-From-Feedback} and \emph{HelpSteer2} demonstrate our method significantly outperforms uniform allocation, reducing worst-case estimation error while maintaining identical budgets. Our work establishes a theoretical foundation for efficient LLM evaluation with practical implications for AI safety, model alignment, and automated assessment at scale.
\end{abstract}

%% file: intro.tex
\vspace{-10pt}
\section{Introduction}
\label{sec:intro}
\vspace{-5pt}
\let\thefootnote\relax\footnotetext{\hspace{-20pt}$^1$ Corresponding author email: aadirupa.saha@gmail.com}
Large language models have revolutionized AI evaluation. Instead of relying solely on expensive human annotators or rigid automated metrics, practitioners now employ LLM-as-a-judge~\citep{zheng2024judging,dubois2024alpacafarm}—querying powerful models like GPT-4 or Claude to assess response quality. Consider a practical scenario: evaluating 10,000 prompt-response pairs for model fine-tuning. Human evaluation at \$5 per judgment costs \$50,000 per iteration—prohibitively expensive. LLM judges offer a compelling alternative at\$0.01 or less per evaluation, enabling scalable quality assessment for supervised fine-tuning, post-training alignment, prompt optimization, A/B testing, and model calibration~\citep{ouyang2022training,bai2022training}.

LLMs are now widely deployed not just for generating content but also for evaluating it, an approach referred to as LLM-as-a-judge \citep{kim2024prometheus,trivedi2024selfrationalizationimprovesllmfinegrained,gu2025surveyllmasajudge,Zhang2024,thakur2025judgingjudgesevaluatingalignment,zhu2025judgelmfinetunedlargelanguage}. This method leverages the natural language understanding of LLMs to replicate human-like evaluations, providing a scalable, economical, and consistent alternative to human annotation. An LLM judge can produce various forms of feedback: natural language critiques, numerical scores, or comparative preferences between options. Often, both textual assessments (rationale) and numeric ratings (scores) are generated together, and the score variation directly depends on the rationale \citep{kim2024prometheus}. 

However, LLM judgments are inherently stochastic. Querying the same (prompt, response) pairs multiple times yields different scores due to sampling randomness and temperature settings. This variance is highly heterogeneous: a factual question like ``\emph{What is $2+2$?}'' might produce consistent scores (say, with score variance $\sigma^2 \approx 0.0001$), while subjective queries like ``\emph{What makes effective leadership?}'' could generate highly variable evaluations ($\sigma^2 \approx 10$). For the first pair, a single query suffices; for the second, many samples are needed for reliable estimates. Yet LLM queries still cost money and consume computational budgets. This raises a fundamental resource allocation problem: \emph{given a fixed budget of $B$ queries across $K$ prompt-response pairs, how should we distribute queries to estimate all scores accurately?} Due to the same reason, the ``LLM-as-Judge" evaluation paradigm faces criticism, as the judgments/scores produced by LLM judges frequently fail to align reliably with human evaluations \citep{chiang-lee-2023-large,gehrmann2023repairing} due to score variability.

Existing work on LLM evaluation has made significant progress on complementary problems. Researchers have developed better prompting strategies~\citep{zheng2024judging,liu2023geval}, mitigated position and length biases~\citep{wang2023chatgpt,zheng2024judging}, and benchmarked judge-human agreement~\citep{liu2023geval,dubois2024alpacafarm}. However, these efforts universally assume \emph{uniform allocation}—querying each pair the same number of times—or ignore the allocation problem entirely. Recent work on LLM uncertainty~\citep{kuhn2023semantic,xiong2023can} focuses on confidence calibration for individual predictions, not budget-constrained multi-instance evaluation. Reward modeling literature~\citep{ouyang2022training,wang2024helpsteer2} collects preference data but does not address optimal query allocation given variance heterogeneity. To our knowledge, \emph{no prior work addresses optimal query allocation under fixed budgets for confidence estimation in LLM judges}. This is a critical gap! A classical baseline like \emph{Uniform Allocation} is provably suboptimal when variances differ, wasting queries on easy pairs while under-sampling difficult ones.

We fill this gap by formalizing LLM judge evaluation as a variance-adaptive resource allocation problem through a multi-armed bandit (MAB) lens: 

\textbf{Informal Problem Statement.} Given $K$ (prompt, response) pairs $\{(q_i, a_i)\}_{i=1}^K$, each associated with a true unknown score $s_i \in \R$, and a fixed query budget $B$, the goal is to estimate the scores $\hat{s}_i$ for all $i \in [K]$.


The objective is to minimize $\|s - \hat{s}\|_\infty$ subject to the budget constraint $B$. In other words, we seek to allocate the total budget $B$ such that the worst-case estimation error $\max_{i \in [K]} |s_i - \hat{s}_i|$ is minimized (details in \cref{sec:prob}). 

\textbf{Contributions}
 of this work are multifold:
\vspace{-10pt}

\begin{itemize}[leftmargin=13pt,noitemsep]
   \item  \noindent\textbf{(C1). Problem Formulation.} Our first contribution lies in formally modeling the LLM judge evaluation as a multi-armed bandit problem where each prompt-response pair is an arm with unknown mean score $s_i$ and variance $\sigma_i^2$, and the goal is to minimize worst-case estimation error (WEC) $\max_{i \in [K]} |s_i - \hat{s}_i|$ under budget constraint $B$ (\cref{sec:prob}).
\vspace{2pt}
\item \noindent\textbf{(C2). Algorithm Design (\cref{sec:alg_known,sec:alg_unknown}).} We develop variance-adaptive allocation strategies for both known and unknown variance settings. For known variances, \algk\ (\cref{alg:algk}) implements a greedy allocation rule that sequentially selects arms based on their respective standard error, i.e. $\arg\max_{i \in [K]} \sigma_i^2 / n_i(t)$, where $n_i(t)$ is the number of times arm $i$ is pulled till time $t$, \emph{naturally balancing high-variance arms with under-sampled arms}. For the realistic unknown-variance setting, \algu\ (\cref{alg:algu}) proceeds in two phases: initial uniform exploration over $t_0 K$ queries to estimate variances, $\{\hat{\sigma}_i^2\}_{i=1}^K$ (\cref{eq:est_var}), up to a certain accuracy. The next phase follows a similar adaptive greedy allocation, same as the ``known-variance case", except now at each time $t$ we replace the true variance $\sigma_i$s with its upper confidence bounds $\ucb{i}{t}$ (\cref{eq:ucb_var}). Both algorithms maintain $O(1)$ computational complexity per query with efficient online updates, making them practical for large-scale evaluations.
\vspace{2pt}
\item \noindent\textbf{(C3). Theoretical Analysis (\cref{sec:theory_k,sec:theory_u}).} We establish near-optimal sample complexity guarantees for both algorithms. For \algk\ with known variances, \cref{thm:known} proves that with probability at least $1-\delta$, the worst-case error satisfies
$
\max_{i \in [K]} |s_i - \hat{s}_i| = \tilde{O}\left(\sqrt{\frac{\sum_{i=1}^K \sigma_i^2}{B}}\right),
$
where $\tilde{O}(\cdot)$ hides logarithmic factors. For \algu\ with unknown variances, \cref{thm:unknown} establishes the same rate with additional logarithmic overhead in the exploration phase, showing that variance estimation incurs negligible cost when $t_0 = \Theta(\log K)$. Our analysis leverages Bernstein-type sub-Gaussianity concentration inequalities (\cref{thm:subgauss_ci}) and novel variance estimation bounds (\cref{lem:var_conc}) adapted to the sequential allocation setting, with extensions to sub-Gaussian and heavy-tailed noise distributions (\cref{rem:noise_unknown}).
    \vspace{2pt}
    \item \noindent\textbf{(C4). Empirical Validation on Real-World LLMs.} We validate our approach on four metrics in HelpSteer2 dataset \citep{wang2024helpsteer2}. The approach performs better than uniform allocation and approaches the performance of variance-based allocation, which is optimal but not implementable because the score variances are unknown. We also show that the improvement in our metric, worst-case absolute error in the estimate of the mean score, improves correlation between the predicted mean score and human ratings. We consider three most popular measures of correlation: Pearson's $r$, Spearman's $\rho$, and Kendall's $\tau$. Our results confirm heterogeneous variance across evaluation pairs is prevalent in practice and adaptive allocation yields substantial improvements in estimation accuracy (\cref{sec:experiments}).
\vspace{-5pt}
\item \noindent\textbf{(C5). Insights on Cost Savings.} Our approach results in significant cost (budget) savings. In particular, it achieves the same worst-case absolute error of estimated mean scores, but with almost half the sample-complexity to that required by the closest implementable baselines. We discuss this ``savings" in more detail in empirical evaluation \cref{sec:experiments}.
\end{itemize}

\vspace{-5pt}
\textbf{Related Work.} Since its introduction, LLM as a judge \citep{zheng23judging} has become the facto standard for LLM evaluation \citep{kim2024prometheus,trivedi2024selfrationalizationimprovesllmfinegrained,gu2025surveyllmasajudge,zhang2024scalingmeetsllmfinetuning,thakur2025judgingjudgesevaluatingalignment,zhu2025judgelmfinetunedlargelanguage}. The goal of all of these works, and many others, is to build an LLM judge that performs well for each prompt-response pair. The main novelty in our work is a realization that the score of each prompt-response pair is a random variable, which depends on the rationale of the judge. By querying the judge multiple times, we can get a better estimate of each mean score. When we have multiple prompt-response pairs, a natural question to ask is how to query the judge to get the most accurate mean score estimate for all prompt-response pairs? We precisely address this question in this work, as motivated above.


%% file: problem.tex
\vspace{-10pt}
\section{Problem Setup}
\label{sec:prob}
\vspace{-5pt}
\begin{table*}[!]
  \scriptsize
  \centering
  \begin{tabular}{l|l}
    Prompt & in a chat, can I feed chatGPT informations so it can answer questions better? \\
    Response & Yes, you can feed ChatGPT information to help it answer questions better. This is called "training" the model, and it \\
    & involves providing it with examples of good and bad responses to questions. You can also use tools like GPT-J's "Teacher" \\
    & mode to provide more detailed feedback on its responses. \\
    Rationale & The response is extremely helpful and clearly explains the concept of training a model like ChatGPT, including the use of \\
    & tools like GPT-J's "Teacher" mode. It directly addresses the user's question and provides relevant information, making it \\
    & easy to understand. The response does not miss any essential details, and its content is accurate and up-to-date. Overall, \\
    & the response is well-written, concise, and effectively conveys the desired information \\
    Score & 4
  \end{tabular}
  \vspace{-5pt}
  \caption{Prompt, response, rationale for its evaluation, and score for one example in the HelpSteer2 dataset \cite{wang2024helpsteer2}.}
  \label{tab:prompt respose}
\end{table*}
\vspace{-5pt}
\textbf{Notation.} The set $\set{1, \dots, n}$ is denoted by $[n]$, for any $n \in \N$. The indicator $\1(E)$ denotes that event $E$ occurs. We use boldface letters to denote vectors. $\cN(\mu,\sigma^2)$, $\cS_{\cG}(\mu,\sigma^2)$ respectively denote Gaussian and sub-Gaussian distribution with mean $\mu \in \R$ and variance (or more generally sub-Gaussianity parameter) $\sigma^2 \in \R_+$ \cite[Chapter 5]{lattimore19bandit}, \cite{vershynin2018high}.
\vspace{-10pt}
\paragraph{Problem Setting.} Consider a set of $K$ evaluation instances $\{(q_i, a_i)\}_{i=1}^K$, where $q_i$ represents an LLM prompt (i.e., a query) and $a_i$ denotes the corresponding response (a.k.a. the answer to the corresponding prompt) to be evaluated.

For each pair $(q_i, a_i)$, there exists a true quality score $s_i \in [0, M]$ where $M \in \mathbb{R}_+$. When we query an LLM judge to evaluate pair $i \in [K]$, we observe a noisy score: $X_{i} = s_i + \epsilon_i$, where $X_{i}$ is the stochastic (noisy) score evaluation of the pair $i$ (by the judge-LLM), and $\epsilon_i$ represents the evaluation noise, s.t. $\mathbb{E}[\epsilon_i] = 0$ and $\text{Var}(\epsilon_i) \leq \sigma_i^2$. We show an example of prompt, response, rationale for its evaluation, and score in \cref{tab:prompt respose}. The example is from the HelpSteer2 dataset \cite{wang2024helpsteer2}, which we experiment with in \cref{sec:experiments}. 
The variance bound $\sigma_i^2$ captures the inherent uncertainty in evaluating pair $i$ and may vary significantly across different pairs due to factors such as query complexity, answer ambiguity, and subjective interpretation requirements.

\vspace{-5pt}
\paragraph{Budget Constraint Optimal Allocation}

We are given a fixed computational budget $B \in \mathbb{N}$, representing the total number of LLM queries available. An allocation strategy determines how to distribute this budget across the $K$ pairs, resulting in allocation $\mathcal{\cB} = (n_1(\cB), n_2(\cB), \ldots, n_K(\cB))$, where $n_i(\cB)$ denotes the number of queries allocated to pair $i$, s.t. $\sum_{i=1}^K n_i(\cB) = B$. Let $\mathbb B(K,B)$ denote the set of all possible allocations of $K$-pairs in a fixed budget $B$. For each pair $i \in [K]$, now one can compute:  
\vspace{-5pt}
\begin{align}
\label{eq:est_sc}
\hat{s}_i(\cB) = \frac{1}{n_i} \sum_{j=1}^{n_i} X_{i,j}, ~~ \cB \in \mathbb B(K,B)
\end{align}
\vspace{-5pt}
the estimated mean judgement score of the pair $i$.
\subsection{Objective and Performance Metric}
\label{sec:obj}

\begin{defn}[Optimal Query Allocation for LLM-as-Judge]
Given $K$ evaluation pairs $\{(q_i, a_i)\}_{i=1}^K$ with unknown (true) scores $\{s_i\}_{i=1}^K$ and a fixed budget $B$, an \emph{Optimal Query Allocation} strategy produces an allocation $\mathcal{B}^* = (n_1(\cB^*),\ldots, n_K(\cB^*)) \in \mathbb B(K,B)$, s.t.:
\begin{equation*}
\mathcal{B}^* = \arg\min_{\cB \in \mathbb B(K,B)} \mathbb{E}\left[\max_{i \in [K]} |s_i - \hat{s}_i(\mathcal{B})|\right]
\end{equation*}
where the expectation is taken over the randomness in the LLM evaluations over $B$ queries.

\end{defn}

\textbf{Objective:} Given a fixed budget $B$, our goal is to find a `\emph{near optimal allocation}' $\mathcal{B}$ to minimize the \emph{Worst-Case Estimation Error} (WCE) across all pairs:
\begin{align}
\label{eq:wce}
\norm{s_i - \hat{s}_i(\mathcal{B})}_\infty = \max_{i \in [K]} |s_i - \hat{s}_i(\cB)|.    
\vspace{-17pt}
\end{align}
Note the above minimax objective ensures that no single (prompt, response) pair is poorly estimated, which is crucial for reliable system-wide evaluation. 


%% file: algo_known.tex
\section{Warm-Up: Optimal Allocation with Known Variance}
\label{sec:alg_known}
\vspace{-5pt}
In this section, we present our primary algorithmic approach for finding the optimal allocation under a fixed budget setting. For simplicity, we initially assume the score variances $\sigma_1^2, \ldots, \sigma_K^2$ are known, which is further relaxed in \cref{sec:alg_unknown}, where we present our most general allocation algorithm for the practical unknown variance setting.

\vspace{0pt}
\begin{algorithm}[H]
  \caption{\algk: Resource Optimization via Budget-aware INference.}
  \label{alg:algk}
  \begin{algorithmic}[1]
    \State \textbf{Input:} Arm set $[K]$; Query-Budget $B$; 
    Arm-Variances $\sigma_i$, for all $i \in [K]$. 
    \State \textbf{Init:} Pull count of Arm-$i$, $n_i(0) = 0, ~\forall i \in [K]$.
    \For{$t = 1,\ldots B$}
     \State Pull Arm-$i_t$ s.t. 
     $i_t = \argmax_{i \in [K]} \frac{\sigma_{i}^2}{n_{i}(t-1)}$ 
      \State Receive score feedback $X_{i_t}$ for Arm-$i_t$
      \State Update $n_{i}(t) \leftarrow n_{i}(t-1) + \1(i_t = i)$,  $\forall i \in [K]$
    \EndFor
    \State Set $n_i(\cR) = n_i(B)$. Final resource allocation $\cR = (n_1(\cR),\ldots,n_K(\cR))$, and 
    estimated score of Arm-$i$:
       \[
        \hs_i(\mathcal{R}):= \frac{1}{n_i(\cR)}\sum_{\ell = 1}^{B} \1(i_\ell = i)X_{i_\ell}, ~\forall i \in [K]
       \]
    \State \textbf{Output:} Estimated score $\hs_i(\cR)$ for Arm-$i \in [K]$
  \end{algorithmic}
\end{algorithm}
\vspace{0pt}

We begin by noting a striking connection between our problem (\cref{sec:obj}) and the fixed-budget MAB literature \cite{karnin13almost,jun2016top,audibert2010best}. However, the usual fixed budget MAB algorithms are often curated for a \emph{different the best-arm identification (BAI) objective, in contrast to our WCE objective as described \cref{eq:wce}}. To this end, we denote each (prompt, response) pair $i$ as an arm of the $K$-armed bandit, referred to as Arm-$i$. 

Note, by our problem formulation in \cref{sec:prob}, this directly implies that the true underlying mean reward/score of Arm-$i$ is $s_i$ with observation variance $\sigma_i^2$, in the language of the MAB literature. 
More precisely, at each round $t$, the MAB-algorithm is now supposed to select an arm $i_t \in [K]$ and observes a sampled score $X_t$ (alternative $X_{i_t}$) such that $\mathbb{E}[X_{i_t}] = s_{i_t}$, and $\text{Var}(X_{i_t}) = \sigma_{i_t}^2$. The goal is to minimize the \emph{worst-case estimation error} (WCE), as defined in \cref{eq:wce}.

Towards this goal, we first design \algk\ (\cref{alg:algk}), an optimal budget allocation algorithm that balances the tradeoff between budget $B$ and score variances $(\sigma_1^2,\ldots, \sigma_K^2)$. We further back this up with a formal analysis of its WCE under the fixed budget constraint $B$. Our analysis (\cref{thm:known}) shows that \algk\ achieves a WCE of $O\biggn{\sqrt{\frac{\sum_{i=1}^K \sigma_i^2}{B}}}$.

\vspace{-10pt}
\subsection{\algk: Algorithm Description } 
Our proposed algorithm \algk\ (Resource Optimization via
Budget-aware INference, \cref{alg:algk}) adapts a `\emph{variance-proportional allocation strategy}' based on the greedy selection rule that sequentially pulls arms with highest standard error of the estimates scores. 
More precisely, at each round $t$, it selects the arm $i_t$ that maximizes the ratio $\sigma_i^2 / n_i(t-1)$, $\sigma_i^2$ is the known variance and $n_i(t)$ being the number of pulls of Arm-$i$ till time $t$ (hence $n_i(0) = 0, ~\forall i \in [K]$). Thus our greedy allocation strategy prioritizes high-variance arms (large numerator) and under-sampled arms (small denominator), naturally balancing allocation toward arms requiring more samples for accurate estimation. We denote \algk's allocation strategy by `$\cR$'.

Upon budget exhaustion, \algk\ computes final score estimates $\hat{s}_i(\cR)$ according to \cref{eq:est_sc}. Note, \algk\ has computational complexity $O(K)$ per round, making it highly efficient for large-scale evaluation scenarios.

\subsection{Performance Analysis of \cref{alg:algk}}
\label{sec:theory_k}

We analyze the WCE performance of \algk($\cR$) in this section. We start by noting that, rather astonishingly, our sequential implementation converges to the closed-form allocation $n_i(\cR) \approx \frac{\sigma_i^2 B}{\sum_{j=1}^K \sigma_j^2}$ from classical optimal allocation theory, while maintaining computational efficiency through online updates:

\begin{restatable}[Allocation Profile of \algk]{lem}{optalloc}
\label{lem:opt_alloc} Let 
$
  \lambda_{i}
  = \frac{\sigma_i^2}{\sum_{j \in [K]} \sigma_j^2}.
$ Then \algk\ pulls Arm-$i$ for at least $\myfloor{\lambda_{i}B}$ many times and at most $\myceil{\lambda_{i}B}$ many times, for all $i \in [K]$, i.e. $n_i(\cR) \in \bign{\myfloor{\lambda_{i}B}, \myceil{\lambda_{i}B}}$.
\end{restatable}

The proof of \cref{lem:opt_alloc} is motivated from \cite[Lemma-1]{lalitha2023fixed}, although {we improve their claim by bypassing the restrictive `integer-$\lambda_i$' assumption}. For completeness, the detailed proof is moved to \cref{app:known}. 

The statement of the above lemma is very crucial towards proving the final error bound of \algk$(\cR)$, as the statement quantifies the number of times every arm gets pulled by our strategy, which holds the key to applying the concentration bounds on $\hs_i(\cR)$. 

\begin{thm}[Performance Analysis of \algk]
\label{thm:known}
Assume the noisy score evaluation of Arm-$i$ follows sub-Gaussian distribution with parameter $\sigma_i^2$ (i.e. $\epsilon_i \sim \cS_\cG(0,\sigma_i^2)$). 
Given fixed budget $B > K$, the estimated scores $\hs_i(\cR)$ computed through the allocation rule $\cR$ of \algk\ (\cref{alg:algk}) achieves WCE:
\[
\max_{i \in [K]} \abs{s_i - \hs_i(\cR)} \leq \sqrt{\frac{\sum_{i=1}^K \sigma_i^2}{B}\log\frac{2K}{\delta}},
\]
with high probability at least $(1-\delta)$, for any confidence parameter $\delta \in (0,1]$.
\end{thm}

\begin{proof}[Proof Sketch of \cref{thm:known}]
The key observation of this proof is built on Gaussian concentration, along with the claim of \cref{lem:opt_alloc}. We start by recalling:

\begin{restatable}[sub-Gaussian Concentration-Inequality \cite{lattimore19bandit}]{thm}{subgci}
\label{thm:subgauss_ci}
If $X \sim \cS_\cG(\mu, \sigma^2)$, then for any $\varepsilon \geq 0$, 
\[
\mathbb{P}(\abs{X-\mu} \geq \varepsilon) \leq 2\exp\left(-\frac{\varepsilon^2}{2\sigma^2}\right).
\]
\end{restatable}

Further noting that for any $i \in [K]$, given $n_i(\cR)$ $=n$ (say), $\hs_i(\cR) \sim \cS_\cG\bign{s_i,\frac{\sigma_i^2}{n}}$ \cite[Lemma 5.4]{lattimore19bandit}, we have
\begin{align}
\label{eq:oct25}
\vspace{-5pt}
    \mathbb{P}(\abs{\hs_i(\cR)-s_i} \geq \varepsilon) \leq 2\exp\left(-\frac{n\varepsilon^2}{2\sigma_i^2}\right).
\vspace{-5pt}
\end{align}
Now \cref{lem:opt_alloc} ensures $\forall i \in [K]$, $n_i(\cR) \geq \myceil{\lambda_i B}$. Then applying \cref{eq:oct25} and taking an union bound over all arms $i \in [K]$, this immediately yields, with probability at least $(1-\delta)$,
\begin{align}
\label{eq:scr_conc}
\vspace{-10pt}
   \max_{i \in [K]} |s_i - \hs_i(\cR)| \leq \sqrt{\frac{2\sum_{i=1}^K \sigma_i^2}{B}\log\frac{2K}{\delta}},
\vspace{-5pt}
\end{align}
which concludes the proof. The detailed analysis is given in \cref{app:known}.
\end{proof}

\vspace{-5pt}
\begin{rem}[Relaxation of the Noise Assumption]
\label{rem:noise_known}
For the interested reader, we emphasize that \emph{our sub-Gaussianity assumption on the noise $\epsilon_i \sim \cS_\cG(0, \sigma_i^2)$, $i \in [K]$, is neither restrictive nor necessary}. First, sub-Gaussianity is a fairly general condition satisfied in most practical settings. Indeed, any bounded random variable $X \in [a,b]$ with $a,b \in \mathbb{R}$ is sub-Gaussian with parameter $\sigma^2 = \frac{(b-a)^2}{4}$ \cite{hoeffding1963probability,wainwright2015basic,boucheron2003concentration}, and boundedness holds in nearly all practical applications.

\emph{More importantly, our proposed algorithm \algk\ (\cref{alg:algu}) is \emph{not} tied to the sub-Gaussianity assumption and works for any zero-mean noise with bounded variances $\{\sigma_i\}_{i \in [K]}$. The distributional assumption only affects the final WCE analysis---specifically, the concentration rate of the estimated scores $\widehat{s}_i(\mathcal{R})$ as derived in \cref{eq:scr_conc} via sub-Gaussian concentration (\cref{thm:subgauss_ci}). If the stochastic noise followed a different distribution---for instance, sub-Exponential, heavy-tailed, sub-Weibull \cite{boucheron13concentration}, or in fact any distribution with known concentration properties---one only needs to substitute the appropriate concentration inequality in place of \cref{thm:subgauss_ci} to obtain corresponding WCE guarantees, keeping the rest of the proof unchanged. This modularity underscores both the generality of our approach and its broad practical applicability.}
\end{rem}

%% file: algo_unknown.tex
\section{Main Algorithm: Near-Optimal Allocation with Unknown Variance}
\label{sec:alg_unknown}

We now address the most general case of the problem (recall the setting in \cref{sec:prob}) with unknown variances $\sigma_1^2, \ldots, \sigma_K^2$. The core algorithmic idea remains the same as that of \algk\ (\cref{sec:alg_known}), i.e. one must balance the tradeoff between budget and variance when allocating queries across arms. However, since the arm variances are unknown, we now estimate them from past observations and allocate resources adaptively to match a near-optimal allocation in the sense of \cref{lem:opt_alloc}. 

Our complete algorithm, \algu\ (presented in \cref{alg:algu}), employs an optimistic (UCB) variance estimation strategy that guides the allocation decisions. We analyze its WCE performance in \cref{thm:unknown}, where the key technical challenge is establishing the convergence of the variance estimates. Surprisingly, despite not knowing the arm variances in advance, our analysis shows that \algu\ achieves performance provably competitive with \algk, wherein lies the novelty of our method and analysis.
\begin{algorithm}[h]
  \caption{\textbf{\algu}: \textbf{\algk} with Hidden resources.}
  \label{alg:algu}
  \begin{algorithmic}[1]
    \State \textbf{Input:} Arm set $[K]$; Query-Budget $B$; Exploration parameter $t_0$;
    \State \textbf{Init-Exploration:} Pull each arm-$i \in [K]$ for $t_0$ rounds.   
    \For{$t = t_0K + 1, \ldots, B$}
      \State Pull Arm-$i_t$ s.t. 
     $i_t = \argmax_{i \in [K]} \frac{\ucb{i}{t-1}}{n_{i}(t-1)}$ 
      \State Receive score feedback $X_{t}$ for Arm-$i_t$
      \State Update $~\forall i \in [K]$: \textbullet~ $n_{i}(t) \leftarrow n_{i}(t-1) + \1(i_t = i)$  
      \State ~\textbullet~ $\hat \sigma_{i}(t)^2$, $\ucb{i}{t-1}$ using \cref{eq:est_var,eq:ucb_var} 
    \EndFor
    \State Resource allocation $\cR_\cH = (n_1(B),\ldots,n_K(B))$, and 
    estimated score of Arm-$i$:
    \vspace{-5pt}
       \[
        \hs_i(\cR_\cH):= \frac{1}{n_i(B)}\sum_{\ell = 1}^{B} \1(i_\ell = i)X_\ell, ~\forall i \in [K].
       \] 
       \vspace{-10pt}
    \State \textbf{Output:} Estimated score $\hs_i(\cR_\cH)$ for Arm-$i \in [K]$
  \end{algorithmic}
\end{algorithm}
\vspace{-8pt}
\subsection{\algu: Algorithm Description}
\label{sec:alg_unknown_desc}

We call our proposed algorithm \algu\ (\cref{alg:algu})---\emph{ROBIN with Hidden Resources}---staying true to the spirit of the original ``Robin Hood", who famously redistributed ``resources" (query budget) to the ``neediest" (arms with highest variance)! 

As motivated above, the structure of \algu\ essentially maintains the same idea as that of \algk\ (\cref{alg:algk}), except it now have to use an optimistic (UCB) of the arm variances $\ucb{i}{t}$ and pull the arms adaptively with highest standard estimation errors. Precisely, we compute the estimated variance at time $t$ of Arm-$i$ as:
\begin{align}
\label{eq:est_var}
    \hat \sigma_i(t)^2 := \frac{1}{n_i(t)} \sum_{\ell = 1}^t\1(i_\ell = 1)\bign{X_{i_\ell} - \hs_j(t)}^2,
\end{align} 
where, as before,
$n_i(t)$ denotes the number of times Arm-$i$ is pulled in $t \in [T]$ rounds, and $\hs_i(t)$ denotes the estimated mean of Arm-$i$ at time $t$, defined as:
$    \hs_i(t) := \frac{1}{n_i(t)} \sum_{\tau = 1}^t\1(i_\tau = i){X_\tau}$.  
%
We further define the UCB estimate of the variance $\sigma_i^2$ of Arm-$i$ at time $t$ as:
\begin{align}
\label{eq:ucb_var}
\ucb{i}{t}:= \frac{\hat{\sigma}_{i}(t)^2}{1 - \sqrt{\frac{4\log(4KB/ \delta)}{n_i(t)}}},
\end{align}
where $\delta \in (0,1]$ is the error (failure) probability parameter. Here the precise form of the $\ucb{i}{t}$ is not very important (which depends on the type of concentration inequality used to bound our estimated variance $\hat \sigma_i(t)^2$, e.g. \cref{lem:var_conc}). But on a high level it satisfies all the desired properties (i) \emph{Concentrates with increasing number of arm pulls}, as clearly the $\ucb{i}{t}$ shrikes with increasing pulls of Arm-$i$ $n_i(t)$. Moreover, it yields a (ii) \emph{valid upper confidence bound (UCB)} with high probability as proved in \cref{lem:var_conc}.


\textbf{Key Algorithmic Ideas. } \algu\ (\cref{alg:algu}) operates in two phases. \textbf{(i)} \textit{Phase 1 (Init-Exploration):} Uniform exploration pulls each arm $i \in [K]$ exactly $t_0$ times, consuming $t_0 K$ queries to obtain initial estimates $\{\hs_i, \hat{\sigma}_i^2\}$ for all arms---note, this in fact makes the denominator in the UCB expression (\cref{eq:ucb_var}) positive and valid. \textbf{(ii)} \textit{Phase 2 (Lines 3-7):} After the initial exploration, the adaptive allocation strategy of \algu\ sequentially assigns the remaining $B - t_0K$ queries by selecting at each round $t$ the arm maximizing $\frac{\ucb{i}{t-1}}{n_i(t-1)}$, where $\ucb{i}{t}$ is the UCB on the estimated variance as defined in \cref{eq:ucb_var}. The intuition is same as that used in \algk\ (\cref{alg:algk}), except now in the absence of the true knowledge of $\{\sigma_i^2\}_{i = 1}^K$, we replace the numerator by its optimistic (UCB) estimate $\ucb{i}{t}$.

This selection rule balances exploration and exploitation: arms with high estimated variance (large $\ucb{i}{t}$) or few samples (small $n_i$) receive priority. Using the upper confidence bound $\ucb{i}{t}$ rather than empirical variance $\hat{\sigma}_i^2$ provides optimistic estimates that prevent premature under-sampling. After each query, the algorithm updates $n_{i}(t)$, $\hs_{i}(t)$, $\hat{\sigma}_{i}(t)^2$, and $\ucb{i}{t}$ as explained above in just $O(K)$ complexity per round $t$. We denote \algu's allocation strategy by `$\cR_\cH$'.

After $B$ queries, \algk\ computes final estimates $\hat{s}_i(\cR_\cH)$ according to \cref{eq:est_sc}. In the next section, we analyze the WCE performance of \algu, which rather surprisingly achieves the same performance as that of \algk, despite the lack of access to the true variances $(\sigma_1^2,\ldots, \sigma_K^2)$---thanks to the sharp concentration of our UCB estimates where the heart of this analysis. 
\vspace{-10pt}
\subsection{Performance Analysis of \cref{alg:algu}}
\label{sec:theory_u}

\vspace{-5pt}
We analyze the WCE score of \algu\ in this section. Lets start with the key final result followed by a brief proof sketch based on how the estimated variance scores ($\hat \sigma_i(t)^2, \ucb{t}{i}$) act as a `close and tight-proxy' of the true score-variance $\sigma_i^2$:
\vspace{-0pt}
\begin{restatable}[Performance Analysis of \algu]{thm}{unknown}
\label{thm:unknown}
Consider any fixed budget $B > 16K \log \frac{4}{\delta}$ 
and assume the noisy score evaluation of Arm-$i$ follows Gaussian distribution with variance $\sigma_i^2$ (i.e. $\epsilon_i \sim \cN(0,\sigma_i^2)$). Then for the choice of $t_0 = 16 \log \frac{4KB}{\delta}$,
the estimated scores $\hs_i(\cR)$ computed through the allocation rule $\cR_\cH$ of \algu\ (\cref{alg:algu}) achieves WCE:
\[
\vspace{-6pt}
\max_{i \in [K]} \abs{s_i - \hs_i(\cR_\cH)} \leq O\left(\sqrt{\frac{\sum_{i=1}^K \sigma_i^2}{B}\log\frac{4KB}{\delta}}\right),
\vspace{-1pt}
\]
with high probability at least $(1-\delta)$, for any confidence parameter $\delta \in (0,1]$.
\end{restatable}
\vspace{-4pt}
\begin{proof}[Proof Sketch of \cref{thm:unknown}]
As we emphasized throughout this section, the primary difference between \algu\ and \algk\ lies in the usage of UCB variances $\ucb{i}{t}$ in place of true variance $\sigma_i^2$ in the choice of $i_t$. Given the rest of \algu\ (\cref{alg:algu}) almost follows the same strategy as that of \algk\ (\cref{alg:algk}), and \algk\ already gives $\tO\left(\sqrt{\frac{\sum_{i=1}^K \sigma_i^2}{B}}\right)$ WEC guarantee with high probability (as derived in \cref{thm:known}), the only missing piece of the puzzle of \cref{thm:unknown} lies in showing with high probability for all time steps $t \in [B]$ and every Arm-$i \in [K]$, $\ucb{i}{t}$ sharply approaches \emph{`close to'} $\sigma_i^2$, roughly at the rate of $\tO\biggn{\sigma_i^2\sqrt{\frac{1}{n_i(t)}}}$. 

More formally, the above claim follows owing to the specific choice of $\ucb{i}{t}$ (see \cref{eq:ucb_var}) and the standard concentration rates of estimated Gaussian variances: 

\begin{restatable}[Estimated Variance Concentration]{lem}{varconc}
\label{lem:var_conc} Consider \algu\ is run for budget $B$, and $t_0 = \tau = 4 \log \frac{4KB}{\delta}$. Then for any failure probability $\delta \in (0,1]$:
\vspace{-5pt}
\begin{align*}
  \mathbb P \Bigg (\exists i \in [K], \text{ and } & t \in \big(\tau K, B\big ] \text{ s.t. }
  \\
  &   \abs{\sigma_i^2 - \hat \sigma_i(t)^2} \geq 2\sigma_i^2\sqrt{\frac{\log(\frac{4KB}{\delta})}{n_i(t)}} \Bigg )
  \leq \frac{\delta}{2}.
\end{align*}
\vspace{-15pt}
\end{restatable}
\vspace{-10pt}
The proof of \cref{lem:var_conc} primarily follows from the sharp concentration rate of the Gaussian variance estimates, as analyzed in \cite[4.4]{laurent00adaptive}. The details proof is moved to \cref{app:unknown} in the Appendix.

\cref{cor:var_conc} further shows that with probability at least $(1-\delta/2)$, for all $t \in (Kt_0, B]$ and $i \in [K]$, \cref{lem:var_conc} further implies
$
\sigma_i^2 \leq \ucb{i}{t} \leq 3 \sigma_i^2.
$

Further, thanks to the almost-identical arm-selection rule of \algk\ and \algu\ (modulo, $\sigma_i$ replaced by $\ucb{i}{t}$ in the greedy selection of $i_t$), following a similar line of argument as that of \cref{lem:opt_alloc}, it can be shown that allocation profile $\cR_\cH$ of \algu\ pulls Arm-$i$ for at least $\myfloor{\frac{\lambda_{i}B}{3}}$ many times, for all $i \in [K]$; i.e. $n_i(\cR_\cH) \geq \myfloor{\frac{\lambda_{i}B}{3}}$, with probability at least $1-\delta/2$. More precisely, using \cref{lem:var_conc} and the arm-selection rule of \algu\, we prove that:

\begin{restatable}[Allocation Profile of \algu ~(\cref{alg:algu})]{lem}{allocunkwn}
\label{lem:opt_alloc_unknown} Let 
$
  \lambda_{i}
  = \dfrac{\sigma_i^2}{\sum_{j \in [K]} \sigma_j^2}.
$ Then \algu\ pulls Arm-$i$ for at least $\myfloor{\frac{\lambda_{i}B}{3}}$ many times, for all $i \in [K]$, i.e. $n_i(\cR_\cH) \geq \myfloor{\frac{\lambda_{i}B}{3}}$.
\end{restatable}

Now given \cref{lem:opt_alloc_unknown} which establishes a lower bound on the minimum allocated budget of Arm-$i \in [K]$, applying the same mean-concentration of \cref{thm:subgauss_ci} over all $i \in [K]$ and taking suitable union bounds, one finally obtains that with probability at least $(1-\delta/2)$,
\vspace{-5pt}
\begin{align}
\label{eq:scr_conc}
    \max_{i \in [K]} |s_i - \hs_i(\cR_\cH)| \leq \sqrt{\frac{6\sum_{i=1}^K \sigma_i^2}{B}\log\frac{4K}{\delta}},
\vspace{-10pt}
\end{align}
which concludes the proof. The detailed analysis of all the results are given in \cref{app:unknown}. 
\end{proof}

\begin{rem}[Relaxation of Noise Stochasticity Assumptions]
\label{rem:noise_unknown}
    Once again, similar to \cref{rem:noise_known}, the execution of \algu\ (\cref{alg:algu}) is \emph{not} tied to the Gaussianity assumption of the score noise and works for arbitrary stochastic noise models. Our theoretical WCE analysis happens to invoke Gaussianity to establish concentration bounds on estimated scores $\hs_i(\cR_\cH)$ (\cref{thm:subgauss_ci}) as well as estimated variances $\hat{\sigma}_i(t)^2$ (\cref{lem:var_conc}) for this case (due to unknown variances $\{\sigma_i^2\}_{i \in [K]}$). But, our analysis extends seamlessly to sub-Gaussian, sub-exponential, heavy-tailed, or other stochastic noise distributions~\cite{boucheron13concentration}, provided one substitutes appropriate concentration inequalities in place of \cref{thm:subgauss_ci} and \cref{lem:var_conc} and adjusts the UCB terms $\ucb{i}{t}$ accordingly. This modularity makes our framework highly practical for real-world language models with diverse noise characteristics.
\end{rem}

%% file: expts_new.tex
\vspace{-10pt}
\section{Experiments}
\label{sec:experiments}
\vspace{-10pt}

\begin{table*}[h]
	\vspace{-3mm}
	\centering
	\scalebox{0.8}{
		\setlength{\tabcolsep}{3pt}
		\renewcommand{\arraystretch}{1.0}
		\begin{tabular}{ll*{6}{c}}
			\toprule
			Model&Attribute & 50k\_Uniform & 50k\_\algu & 50k\_\algk & 100k\_Uniform & 100k\_\algu & 100k\_\algk \\
			\midrule
			llama-3-1-8b & complexity  & 0.385±0.011 & 0.324±0.010 & 0.314±0.007 & 0.266±0.006 & 0.233±0.009 & 0.202±0.003 \\
			gpt-4.1-nano & complexity  & 0.437±0.010 & 0.368±0.011 & 0.381±0.011 & 0.320±0.007 & 0.287±0.013 & 0.251±0.007 \\
			llama-3-1-8b & correctness & 0.420±0.015 & 0.296±0.006 & 0.288±0.006 & 0.310±0.012 & 0.219±0.006 & 0.195±0.004 \\
			gpt-4.1-nano & correctness & 0.451±0.012 & 0.300±0.008 & 0.273±0.006 & 0.321±0.007 & 0.249±0.011 & 0.187±0.003 \\
			llama-3-1-8b & helpfulness & 0.383±0.007 & 0.314±0.006 & 0.321±0.007 & 0.265±0.006 & 0.255±0.009 & 0.203±0.004 \\
			gpt-4.1-nano & helpfulness & 0.413±0.016 & 0.265±0.005 & 0.263±0.005 & 0.291±0.012 & 0.220±0.005 & 0.187±0.004 \\
			llama-3-1-8b & verbosity   & 0.523±0.017 & 0.360±0.010 & 0.343±0.007 & 0.369±0.010 & 0.275±0.010 & 0.243±0.007 \\
			gpt-4.1-nano & verbosity   & 0.455±0.012 & 0.342±0.009 & 0.331±0.008 & 0.322±0.008 & 0.289±0.013 & 0.225±0.005 \\
			\bottomrule
		\end{tabular}
	}
	\vspace{-1mm}
	\caption{\small WCE across different configurations. All values are in mean ± 1 std-deviation errors. Column headings are formatted as ``number of queries \_ algorithm used"}
	\label{table:full_comparison}
	\vspace{-1mm}
\end{table*}

In this section, we empirically verify our algorithm's feasibility from three results. First, the statistics of the dataset, which shows the pattern we exploit. Second, a comparison of, \algu, \algk\;and a baseline of uniform allocation. Finally and a correlation of LLM-judge responses and human-responses to show the value of using LLMs and judges. 
\vspace{-0.1in}
\subsection{Experimental Setup}
\label{sec:exp_setup}
\vspace{-0.1in}
\textbf{Dataset}
We evaluate our algorithm on HelpSteer2 \cite{wang2024helpsteer2}, which is a dataset of 20.3 prompt-response pairs evaluated by humans on certain attributes such as ``helpfulness", ``correctness", ``complexity", and ``verbosity". We chose HelpSteer2 because of it's popularity and the fact that it contains human scores for multiple attributes. LLMs trained on HelpSteer2 reach state-of-the-art performance comparable to much larger datasets. Thus demonstrating the high-quality ratings on HelpSteer2. To thoroughly evaluate our algorithm, we determine how well we can approximate the high bar of human ratings on the HelpSteer2 dataset.

In our experiments, we take 1k prompt-response pairs from HelpSteer2, focusing on four attributes: ``helpfulness", ``correctness", ``complexity", and ``verbosity". We then used the LLMs ``Llama 3.1 8b instruct" and ``GPT-4.1 nano" to create a set of 30 ratings per prompt-response pair evaluating the prompt-response on each of the four attributes individually. This creates a total of 30k prompt-response evaluations per attribute. We evaluate our algorithms on this generated data.

\textbf{Simulation} To simulate our algorithms' performance, we allow the algorithm to direct which prompt-response pair to evaluate at each iteration. When a pair is selected, we randomly select one of the 30 evaluations of the prompt-response pairs and return it's evaluation to simulate a judge-LLM evaluating the prompt-response pair. 

\textbf{Algorithms} We implemented \algu\; as proposed in \cref{sec:alg_unknown_desc}. We also specify the value of $\delta$ for each instance of \algu. We also implemented \algk\; as well as a baseline of uniform allocation of compute.

\textbf{Hyperparameter}   
$t_0$ is defined to be $t_0 = 4\ln(1/\delta)$. A lower $\delta$ makes a tighter bound, elaborated in \cref{thm:unknown}, however, it will require the algorithm to behave as uniform allocation for much longer, and hence not leave much room for improvement. Choosing a high value of $\delta$ is also problematic as the UCB fails more frequently, and our allocation of compute is not similar to the optimal variance allocation of compute.

Hence, choosing an appropriate value of $\delta$ is crucial for the algorithm's success. The $\delta$ we chose for our experiments was tuned and selected because it performed well across all our experiments. We also include an example \cref{fig:gpt_50k_delta07} in which a poorly chosen $\delta$ undermines our results.


\vspace{0mm}
\begin{figure}[h]
	\begin{minipage}{0.48\columnwidth}
		\centering
		\includegraphics[trim=10 10 10 10, clip, width=\textwidth]{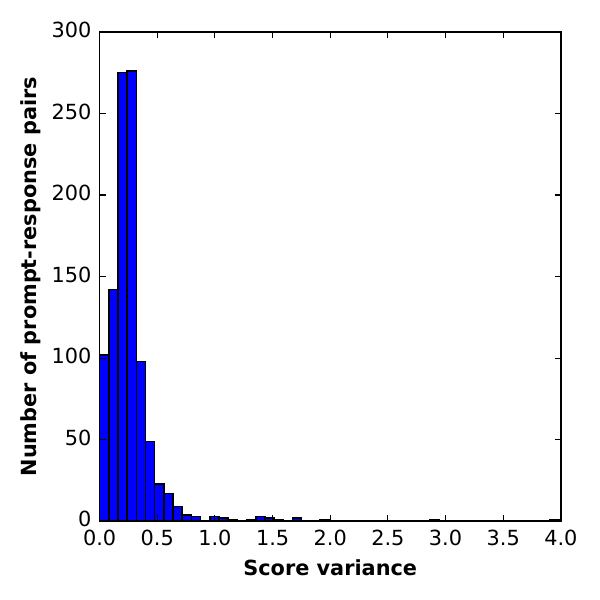}
		\vspace{-6mm}
		\caption{Histogram of score variance of GPT-4.1 nano on helpfulness attribute }
		\label{fig:population_variance}
	\end{minipage}
	\hfill
	\begin{minipage}{0.48\columnwidth}
		\centering
		\includegraphics[trim=10 10 10 10, clip, width=\textwidth]{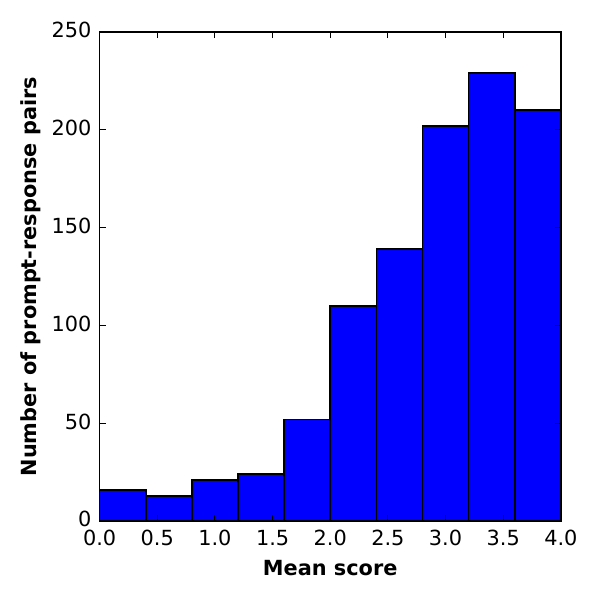}
		\vspace{-6mm}
		\caption{Histogram of mean scores of GPT-4.1 nano on helpfulness attribute }
		\label{fig:mean_scores}
	\end{minipage}
\end{figure}

\textbf{Dataset Statistics}
Here we draw attention to \cref{fig:population_variance} and \cref{fig:mean_scores}. Note that this difference in variation in \cref{fig:population_variance} is exactly what \algk\; exploits and \algu\; approximates to reduce errors faster than uniform allocation. \cref{fig:mean_scores} depicts a histogram of the distribution our algorithms converge to. We choose the same attribute ``helpfulness", evaluated on GPT-4.1 nano, that is used to create Figures \ref{fig:gpt_50k_delta007}, \ref{fig:gpt_100k_delta007}, \ref{fig:gpt_50k_delta07} and \ref{fig:correlation_gpt}.

	
	


\subsection{Inference from the Empirical Results}
\vspace{-0.1in}
This section discusses our individual experiments and what we infer from each experiment. Additional experiments and prompt details are included in \cref{app:expts}.

\begin{figure*}[t!]
	\centering
	\begin{minipage}{0.3\textwidth}
		\centering
		\includegraphics[width=2in]{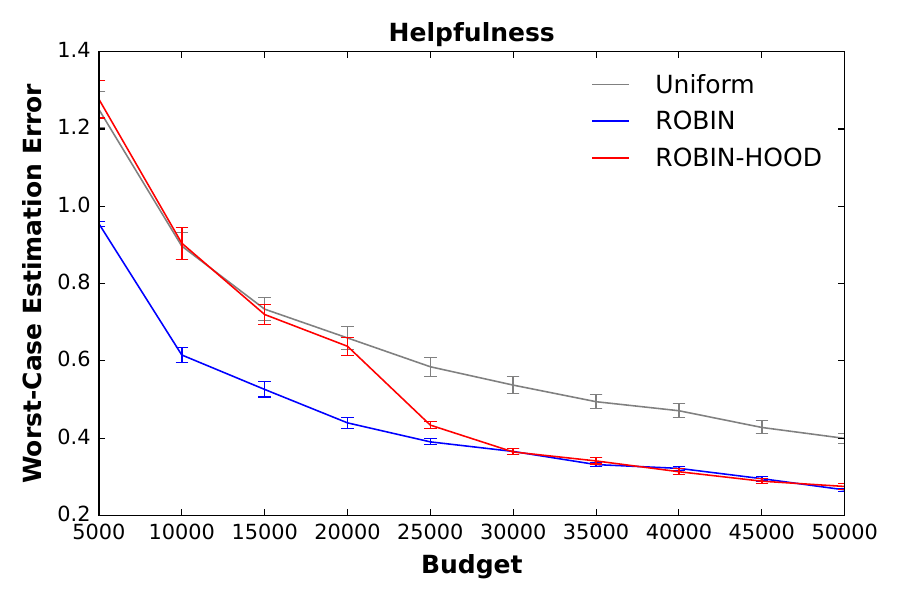}
		\vspace{-0.25in}
		\caption{Maximum error, using GPT-4.1 nano, $\delta$=0.007, warm-up period: 20164 samples}
		\label{fig:gpt_50k_delta007}
	\end{minipage}
	\hfill
	\begin{minipage}{0.3\textwidth}
		\centering
		\includegraphics[width=2in]{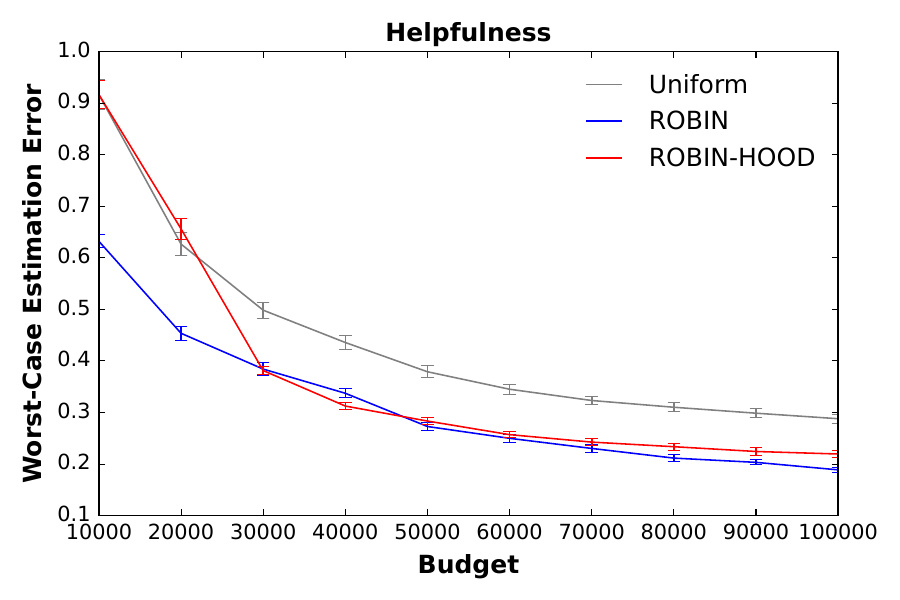}
		\vspace{-0.25in}
		\caption{Maximum error, using GPT-4.1 nano, $\delta$=0.007, warm-up period: 20164 samples}
		\label{fig:gpt_100k_delta007}
	\end{minipage}
	\hfill
	\begin{minipage}{0.3\textwidth}
		\centering
		\includegraphics[width=2in]{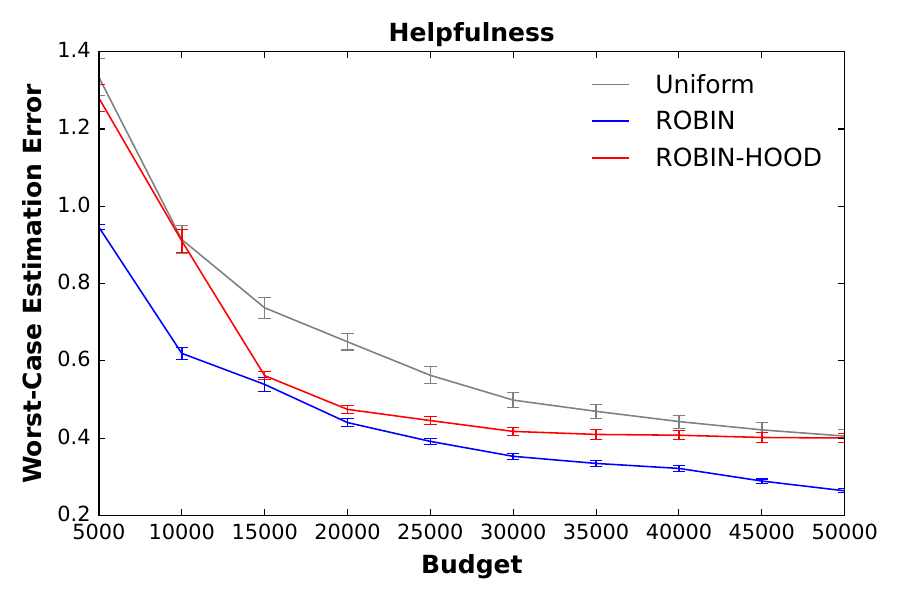}
		\vspace{-0.25in}
		\caption{Maximum error, using GPT-4.1 nano, $\delta$=0.07, warm-up period: 10807 samples}
		\label{fig:gpt_50k_delta07}
	\end{minipage}
	
	\vspace{-5.5mm}
\end{figure*}

\textbf{Figures [\ref{fig:gpt_50k_delta007},\ref{fig:gpt_100k_delta007} and \ref{fig:gpt_50k_delta07}]} we outperform the uniform allocation baseline, on the WCE metric defined in \cref{eq:wce}. Our algorithm is based on an estimate of the variance, so we cannot beat the true variance (which is unknown when querying the LLM-judges). Hence, our algorithm's performance lies between the uniform compute allocation and \algk, which is the true variance-based allocation. Note that as our algorithm allocates compute uniformly until the warm-up period is over, its performance is identical to the Uniform allocation during this period, after which there is a sharp increase in performance. Note that as we choose a higher delta, the performance of \algu\; approaches that of \algk, this is natural as a higher delta implies a tighter upper bound on the variance of each prompt-response pair. 

\textbf{Time Saved} The main motivation of these experiments is to demonstrate how our approach can reduce the time spent evaluating an LLM using a Judge-LLM by approximately half. By comparing \cref{fig:gpt_50k_delta007} and \cref{fig:gpt_100k_delta007} we can see that through uniform allocation of compute we get a maximum error score of about 0.3 in 100k queries, however with \algu, we can get the same error rate in approximately 50k queries. We corroborate this result in \cref{table:full_comparison}.

Also, we include an example of poorly choosing $\delta$ where the WCE plateaus, as observed in \cref{fig:gpt_50k_delta07}. Here at the 50k step mark, we see that the error is virtually identical to the uniform error and will be overtaken by the uniform allocation algorithm. Hence, we see the importance of choosing $\delta$ appropriately.




\textbf{\cref{table:full_comparison}} shows the results of the uniform baseline, \algu, to the baseline and \algk. Here we can see that across all models and attributes in the dataset, there is a statistically significant difference between \algu\;and Uniform compute allocations. We can also see that the \algu's error at the 50k mark is close to the Uniform algorithm's error at the 100k mark across all models and attributes. This further corroborates that we can reduce the time of evaluating an LLM through a Judge-LLM by half. For all experiments, we chose $\delta= 0.007$, resulting in a warm-up period of 20164 queries. All the experiments were averaged across 50 runs and were evaluated on the ``helpfulness" attribute.

\begin{figure}[htbp]
	\begin{minipage}{0.49\columnwidth}
		\centering
		\includegraphics[trim=10 10 10 10, clip, width=\textwidth]{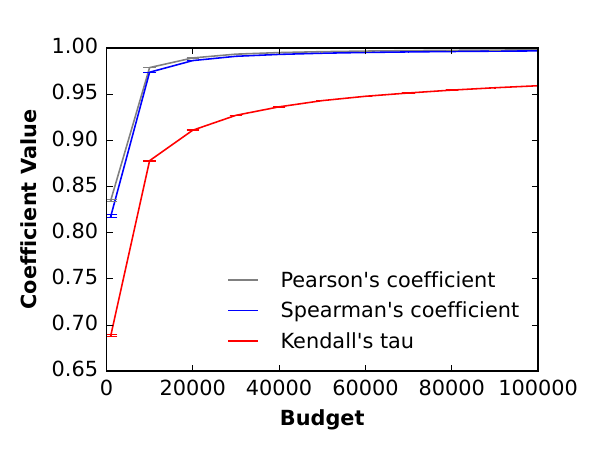}
		\vspace{-6mm}
		\caption{Correlation with human ratings, Llama 3.1 8B Instruct, $\delta$=0.007, warm-up: 20164}
		\label{fig:correlation_llama}
	\end{minipage}
	\hfill
	\begin{minipage}{0.49\columnwidth}
		\centering
		\includegraphics[trim=10 10 10 10, clip, width=\textwidth]{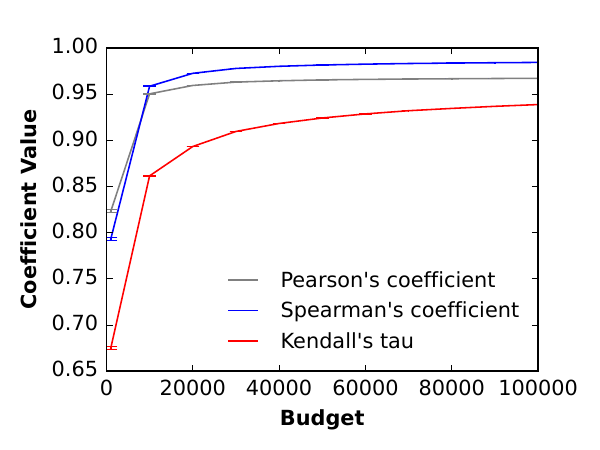}
		\vspace{-6mm}
		\caption{Correlation with human ratings, GPT-4.1 nano, $\delta$=0.007, warm-up: 20164}
		\label{fig:correlation_gpt}
	\end{minipage}
	\vspace{-5mm}
\end{figure}



Our analysis and theoretical results show that we optimize the WCE metric. In the next experiment, we investigate how estimated scores with a lower WCE correlate with human scores better than estimated scores with a higher WCE.

\textbf{Figures \ref{fig:correlation_llama} and \ref{fig:correlation_gpt}} show how the correlation between LLM estimates and human scores rises as the budget increases. This correlation demonstrates that using Judge-LLM provides empirical scores that are highly correlated with human scores, thus validating the approach of using LLMs as judges for other LLMs. We used \algu\; to create empirical ratings that we correlated to human ratings.

%% file: conclusion.tex
\vspace{-7pt}
\section{Conclusions}
\label{sec:concl}
\vspace{-7pt}
We established a principled framework for confidence estimation in LLM-as-judge systems under fixed computational budgets by formalizing the problem as variance-adaptive resource allocation. 
\vspace{-10pt}
\paragraph{Future Directions.} 
Several high-impact extensions merit investigation: \emph{Contextual prediction:} Incorporating features, e.g., query complexity, answer length, semantic embeddings, could enable predictive models that estimate variances with personalization. \emph{Multi-attribute evaluation:} Extending our framework to joint allocation across correlated evaluation dimensions (helpfulness, factuality, coherence) presents opportunities to exploit correlation structure for further efficiency gains. \emph{Active pair selection:} Beyond allocating queries across fixed pairs, strategic selection of \emph{which} pairs to evaluate (e.g., near decision boundaries in model comparison) could also yield order-of-magnitude cost reductions in large-scale scenarios. 


%% file: appendix.tex
\appendix

\onecolumn{

\section*{\centering\Large{Supplementary for \papertitle}}
\vspace*{1cm}

\input{app_known.tex}

\input{app_unknown.tex}


\input{appendix_expts.tex}

\newpage

\input{appendix_prompts.tex}
}

%% file: app_known.tex
\section{Appendix for \cref{sec:alg_known}}
\label{app:known}

\subsection{Proof of \cref{thm:known}}

Further applying the mean-concentration of \cref{thm:subgauss_ci} over all $i \in [K]$ and taking a union bound, one finally obtains that with probability at least $(1-\delta)$,
\begin{align}
\label{eq:scr_conc}
    |s_i - \hs_i(\cR)| \leq \sqrt{\frac{2\sum_{i=1}^K \sigma_i^2}{B}\log\frac{2K}{\delta}}.
\end{align}

This can be shown using

\subgci*

Using \cref{thm:subgauss_ci}, and further noting that for any $i \in [K]$, given $n_i(\cR)$ $=n$ (say), $\hs_i(\cR) \sim \cS_\cG\bign{s_i,\frac{\sigma_i^2}{n}}$ \cite[Lemma 5.4]{lattimore19bandit}, we have
\begin{align}
\label{eq:1}
    \mathbb{P}(\abs{\hs_i(\cR)-s_i} \geq \varepsilon) \leq 2\exp\left(-\frac{n\varepsilon^2}{2\sigma_i^2}\right).
\end{align}

Further setting $2\exp\left(-\frac{n\varepsilon^2}{2\sigma_i^2}\right) \leq \delta/K$, \cref{eq:1} gives:

\begin{align}
    \mathbb{P}\biggn{\abs{\hs_i(\cR)-s_i} \geq \sqrt{\frac{2\sum_{i=1}^K \sigma_i^2}{B}\log\frac{2K}{\delta}}} \leq \delta/K,
\end{align}
as follows noting, 

\begin{align*}
\exp\left(-\frac{n\varepsilon^2}{2\sigma_i^2}\right) \leq \delta/K
\implies
\epsilon^2 \geq \frac{2\sigma_i^2}{n}\ln \frac{K}{\delta};
\end{align*}

which further implies $\epsilon = \sqrt{\frac{2\sum_{i=1}^K \sigma_i^2}{B}\log\frac{2K}{\delta}}$, is a valid choice as $n \geq \myfloor{\frac{\sigma_i^2B}{\sum_{j \in [K]} \sigma_j^2} }$ by \cref{lem:opt_alloc}.

Taking a union bound over all $i \in [K]$, concludes the claim.

\subsection{Proof of \cref{lem:opt_alloc}}

\optalloc*

\begin{proof}


Let for any $x \in \R_+$, let $\rnd(x) \in \N \cup \{0\}$ denotes the closest integer to $x$ such that $\abs{x - \rnd(x)} \leq 0.5$.

Assume the statement is false, and there exists an arm, say Arm-$i$, that was under-pulled, so $n_i(\cR) \leq \rnd(\lambda_i B)-1$. 

But note this immediately implies $\exists j \in [K]$, an over-pulled arm, such that $n_j(\cR) \geq \rnd(\lambda_i B)+1$, as otherwise $\sum_{k \in [K]}n_k(\cR) < B$, which is not possible. 

But for the above situation to occur, assume that $t$ is the minimum time-index such that: 
$n_i(t-1) \leq \rnd(\lambda_i B)-1$ and $n_j(t-1) = \rnd(\lambda_j B)$ and $j$ got pulled again. So at the end of time $t$ it first happens that 
$n_i(t) = n_i(t-1)$ and $n_j(t) = \rnd(\lambda_j B)+1$.

But note at time $t$,
\begin{align*}
    \frac{\sigma_i^2}{n_{i}(t-1)} 
    \geq \frac{\sigma_i^2}{\rnd(\lambda_i B)-1} 
    > \frac{\sigma_i^2}{\lambda_i B}
  = \frac{\sum_{k \in [K]} \sigma_k^2}{B}
  = \frac{\sigma_j^2}{\lambda_j B}
  > \frac{\sigma_j^2}{\rnd(\lambda_j B)+1}
  = \frac{\sigma_j^2}{n_i(t-1)}\,,
\end{align*}
leading to a contradiction that $j$ is pulled at time $t$! This means $j$ cannot be pulled at any such round $t$, or in turn $n_i(\cR) \not\leq \rnd(\lambda_i B)-1$; it has to be that $n_i(\cR) \geq \rnd(\lambda_i B) \leq \myfloor{\lambda_i B}$. 

An exact similar proof by contradiction argument also leads to $n_i(\cR) \leq \rnd(\lambda_i B) \leq \myceil{\lambda_i B}$

This concludes the proof.
\end{proof}

%% file: app_unknown.tex
\section{Appendix for \cref{sec:alg_unknown}}
\label{app:unknown}

\subsection{Proof of Variance Concentration \cref{lem:var_conc} and Implications \cref{cor:var_conc}}

\varconc*

\begin{proof}[Proof of \cref{lem:var_conc}]

We start recalling the sharp concentration bound of the estimated variance of Gaussian random variables from \cite[4.4]{laurent00adaptive}, which is known to guarantee: 

\begin{thm}[Variance Concentration of Gaussian Random Variables \cite{laurent00adaptive}]
\label{thm:gauss_var_CI}
    If $X_1, \ldots X_n$ are $n$ iid draws from $\cN(\mu,\sigma^2)$, and we define $\hat \sigma(n)^2 = \frac{1}{n-1}\sum_{t = 1}^{n}\bign{X_t - \hat\mu(n)}^2$, where $\hat \mu(n): = \sum_{t = 1}^n X_t$, then
\begin{align*}
  \pr{\abs{\sigma_i^2 - {\hat{\sigma}(n)^2}}
  \geq 2 \sigma_i^2\sqrt{\frac{\log(2 / \delta)}{n}}}
  \leq \delta\,.
\end{align*} 
\end{thm}
    
The proof of \cref{lem:var_conc} now follows by directly applying \cref{thm:gauss_var_CI} on $\hat \sigma_i(t)$ for with failure probability $\frac{\delta}{2BK}$, and further taking a union bound over all $t \in [B]$ and $i \in [K]$.
\end{proof}

\begin{cor}[Variance Sandwithcing]
\label{cor:var_conc}
 With probability at least $(1-\delta/2)$, for all $t \in (Kt_0, B]$ and $i \in [K]$:
$
\sigma_i^2 \leq \ucb{i}{t} \leq 3 \sigma_i^2.
$
\end{cor}

\begin{proof}
To prove the first part, note \cref{lem:var_conc} immediately implies that with probability at least $(1-\delta/2)$, for all $t \in (Kt_0, B]$ and $i \in [K]$,
\[
\frac{\hat{\sigma}_{i}(t)^2}{1 + \sqrt{\frac{4\log(4KB/ \delta)}{n_i(t)}}} \leq \sigma_i^2 \leq \frac{\hat{\sigma}_{i}(t)^2}{1 - \sqrt{\frac{4\log(4KB/ \delta)}{n_i(t)}}} ~= \ucb{i}{t},
\]
justifying the indeed $\ucb{i}{t}$ is a high-confidence upper bound of $\sigma_i^2$ (recall \cref{eq:ucb_var}).

 To prove the second part, note that owing to our initial exploration, since $n_i(t) \geq t_0 = 16 \log \frac{4KB}{\delta}$, we can conclude: 

\begin{align*}
    & \frac{\hat{\sigma}_{i}(t)^2}{1 - \sqrt{\frac{4\log(4KB/ \delta)}{n_i(t)}}} \frac{1 - \sqrt{\frac{4\log(4KB/ \delta)}{n_i(t)}}}{1 + \sqrt{\frac{4\log(4KB/ \delta)}{n_i(t)}}} 
    \leq 
    \sigma_i^2  \frac{1 - \sqrt{\frac{4\log(4KB/ \delta)}{n_i(t)}}}{1 - \sqrt{\frac{4\log(4KB/ \delta)}{n_i(t)}}}
    \\
    \implies & \frac{\hat{\sigma}_{i}(t)^2}{1 - \sqrt{\frac{4\log(4KB/ \delta)}{n_i(t)}}}
    \leq \sigma_i^2  \frac{1 + \sqrt{\frac{4\log(4KB/ \delta)}{n_i(t)}}}{1 - \sqrt{\frac{4\log(4KB/ \delta)}{n_i(t)}}}.
\end{align*}

But since $n_i(t) \geq t_0$, by initial exploration, we further get:
\begin{align*}
    & \ucb{i}{t} = \frac{\hat{\sigma}_{i}(t)^2}{1 - \sqrt{\frac{4\log(4KB/ \delta)}{n_i(t)}}}
    \leq \sigma_i^2  \frac{1 + \sqrt{\frac{4\log(4KB/ \delta)}{n_i(t)}}}{1 - \sqrt{\frac{4\log(4KB/ \delta)}{n_i(t)}}}
    \leq 3 \sigma_i^2,
\end{align*}
concluding the result.
\end{proof}

\subsection{Proof of Allocation Profile \cref{lem:opt_alloc_unknown}}

\allocunkwn*

\begin{proof}
Recall the notation $\rnd(x) \in \N \cup \{0\}$ from the proof of \cref{lem:opt_alloc}.

Let's assume the statement of \cref{lem:opt_alloc_unknown} is false, and there exists an arm, say Arm-$i$, that was under-pulled, so $n_i(\cR) \leq \rnd(\frac{\lambda_i B}{3})-1$. 

But note this immediately implies $\exists j \in [K]$, an over-pulled arm, such that $n_j(\cR) \geq \rnd({\lambda_i B})+1$, as otherwise $\sum_{k \in [K]}n_k(\cR) < B$, which is not possible. 

But for the above situation to occur, assume that $t$ is the minimum time-index such that: 
$n_i(t-1) \leq \rnd(\frac{\lambda_i B}{3})-1$ and $n_j(t-1) = \rnd(\lambda_j B)$ and $j$ got pulled again. So at the end of time $t$ it first happens that 
$n_i(t) = n_i(t-1)$ and $n_j(t) = \rnd(\lambda_j B)+1$.

But note at time $t$,
\begin{align*}
    \frac{\ucb{i}{t-1}}{n_{i}(t-1)} \overset{(a)}{\geq} \frac{\sigma_i^2}{n_{i}(t-1)} 
    > \frac{\sigma_i^2}{\rnd(\frac{\lambda_i B}{3})-1} 
    \geq \frac{3\sigma_i^2}{\lambda_i B}
  = \frac{3\sum_{k \in [K]} \sigma_k^2}{B}
  = \frac{3\sigma_j^2}{\lambda_j B}
  \overset{(b)}{\geq} \frac{\ucb{j}{t-1}}{\lambda_j B}
  > \frac{\ucb{j}{t-1}}{\rnd(\lambda_j B)+1}  
  = \frac{\ucb{j}{t-1}}{n_j(t-1)},
\end{align*}
where $(a)$ and $(b)$ follows from \cref{cor:var_conc}.
But that leads to a contradiction that $j$ is pulled at time $t$, since it turns out that at time $t$ 
\[
\frac{\ucb{j}{t-1}}{n_{j}(t-1)} < \frac{\ucb{i}{t-1}}{n_{i}(t-1)}.
\]
This means $j$ cannot be pulled at any such round $t$, or in turn $n_i(\cR) \not\leq \rnd(\frac{\lambda_i B}{3})-1$; it has to be that $n_i(\cR) \geq \rnd(\frac{\lambda_i B}{3}) \leq \myfloor{\frac{\lambda_i B}{3}}$. 
\end{proof}

\subsection{Proof of the Main Theorem \cref{thm:unknown}}

\unknown*

\begin{proof}[Proof of \cref{thm:unknown}]
The key lies in the fact that our allocation strategy of \algu\ (\cref{alg:algu}) ensures each arm gets at least a `fair-share' of the total budget, based on its underlying variance! More formally, for all (prompt-response) pair $i \in [K]$, $n_i(\cR_\cH) \geq \myfloor{\lambda_{i}B/3}$, with probability at least $1-\delta/2$, as proved in \cref{lem:opt_alloc_unknown}.

Further applying the mean-concentration of \cref{thm:subgauss_ci} over all $i \in [K]$, following the same line of concentration argument as given in the proof of \cref{thm:known}, and taking a union bound over all $i \in [K]$ and $t \in [B]$, one finally obtains that with probability at least $(1-\delta/2)$,
\vspace{-5pt}
\begin{align}
\label{eq:scr_conc}
    \max_{i \in [K]} |s_i - \hs_i(\cR_\cH)| \leq 2\sqrt{\frac{\sum_{i=1}^K \sigma_i^2}{B}\log\frac{4K}{\delta}},
\vspace{-10pt}
\end{align}
which concludes the proof. 
\end{proof}

%% file: appendix_expts.tex
\section{Appendix for \cref{sec:experiments}: Additional Experiments}
\label{app:expts}

This section expands upon the evidence in \cref{sec:experiments}. We detail the exact process used to generate the dataset that we used, and demonstrate our algorithms across multiple datasets with different attributes to evaluate LLMs. 

\subsection{Convergence of \cref{table:full_comparison}}

In \cref{fig:errors_full} we show the whole trajectory of each training that we filtered to create \cref{table:full_comparison}. These results show how our WCE metric, defined in \cref{eq:wce}, is reduced under each algorithm across multiple attributes that are commonly used to evaluate LLMs. We compare \algk, \algu, and Uniform allocation of compute. Our results show that in all scenarios \algu\; performs as expected.

\begin{figure*}[t]
	\centering
	\begin{minipage}{0.24\textwidth}
		\centering
		\includegraphics[width=\textwidth]{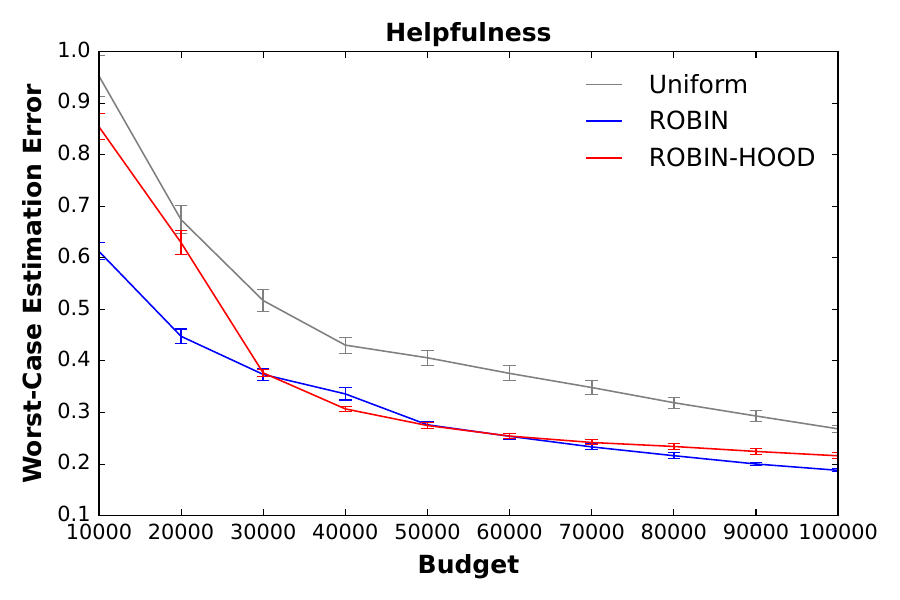}
	\end{minipage}
	\hfill
	\begin{minipage}{0.24\textwidth}
		\centering
		\includegraphics[width=\textwidth]{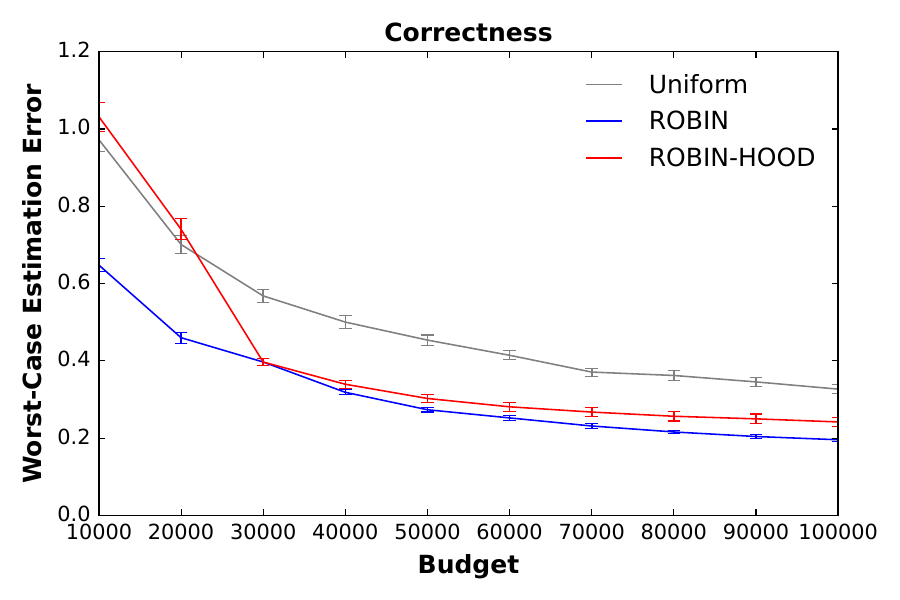}
	\end{minipage}
	\hfill
	\begin{minipage}{0.24\textwidth}
		\centering
		\includegraphics[width=\textwidth]{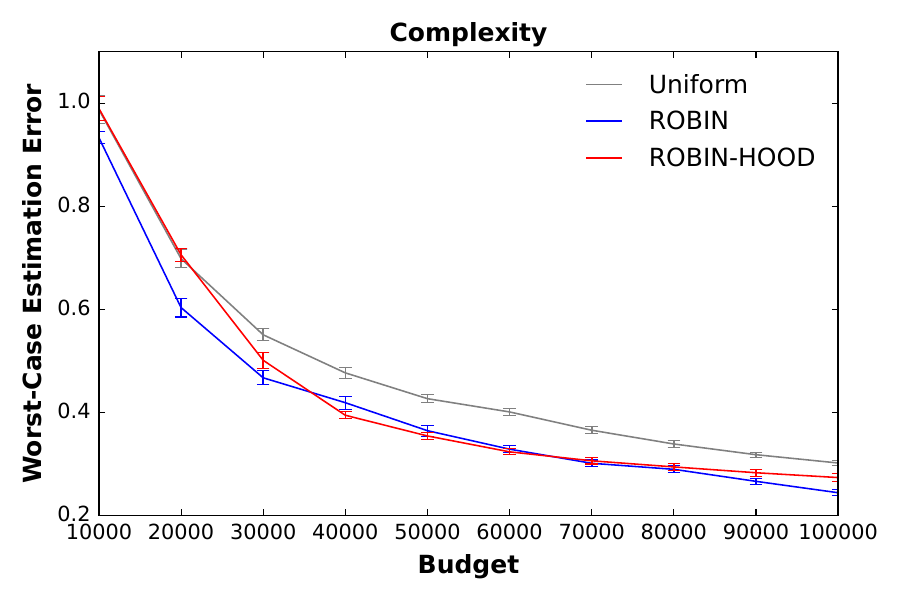}
	\end{minipage}
	\hfill
	\begin{minipage}{0.24\textwidth}
		\centering
		\includegraphics[width=\textwidth]{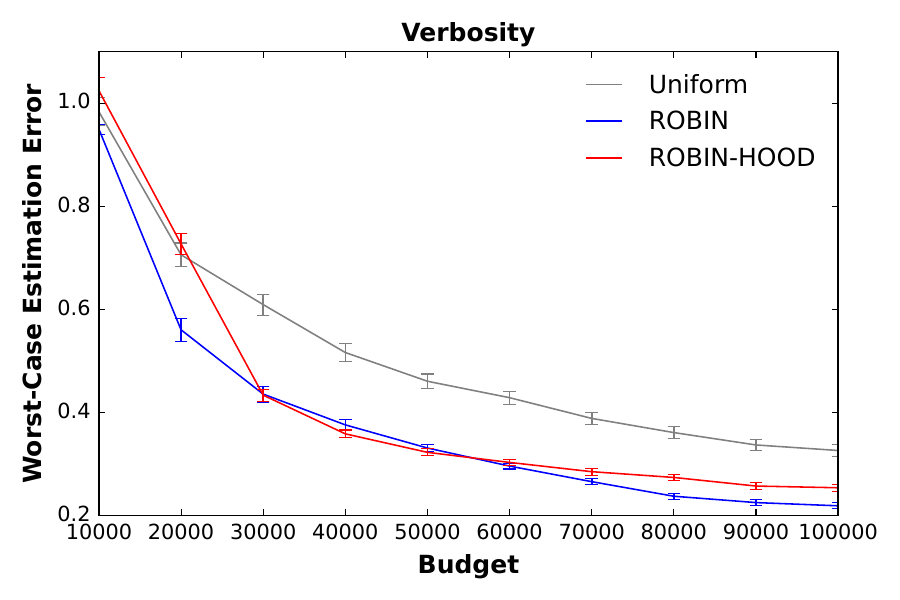}
	\end{minipage}
	
	
	\begin{minipage}{0.24\textwidth}
		\centering
		\includegraphics[width=\textwidth]{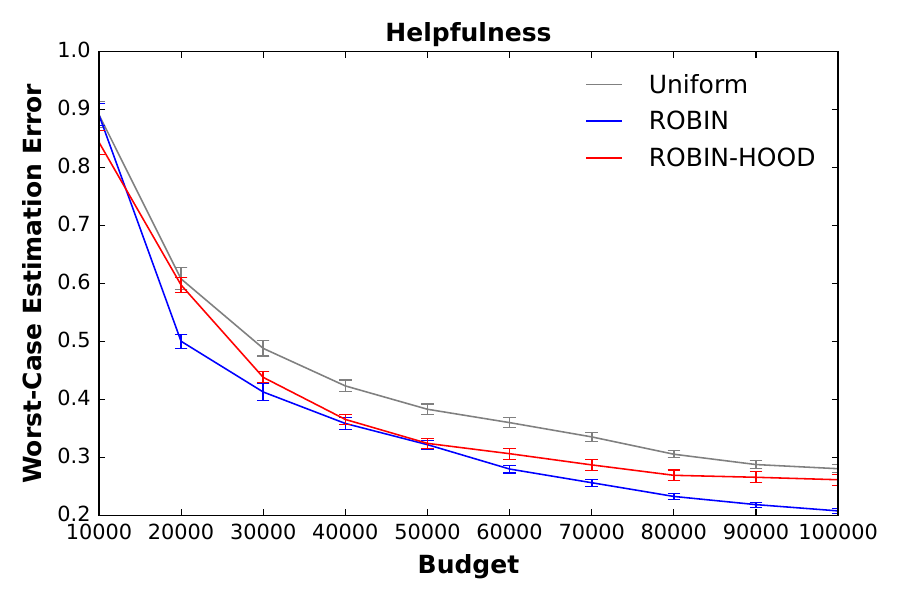}
	\end{minipage}
	\hfill
	\begin{minipage}{0.24\textwidth}
		\centering
		\includegraphics[width=\textwidth]{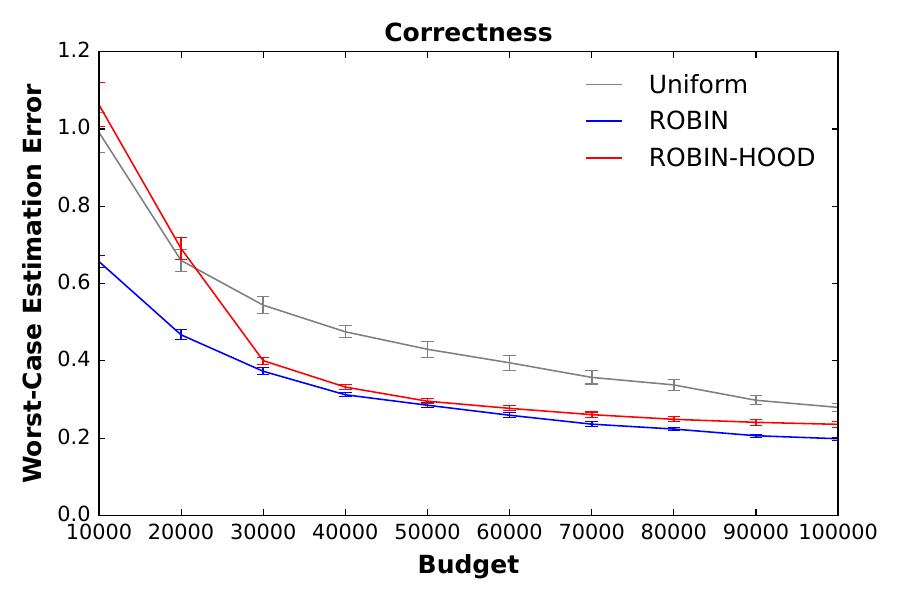}
	\end{minipage}
	\hfill
	\begin{minipage}{0.24\textwidth}
		\centering
		\includegraphics[width=\textwidth]{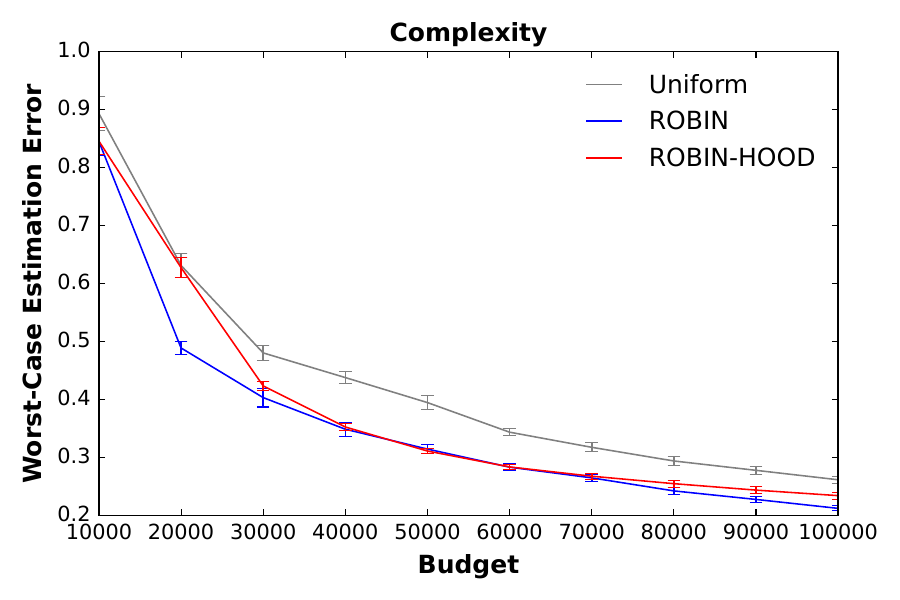}
	\end{minipage}
	\hfill
	\begin{minipage}{0.24\textwidth}
		\centering
		\includegraphics[width=\textwidth]{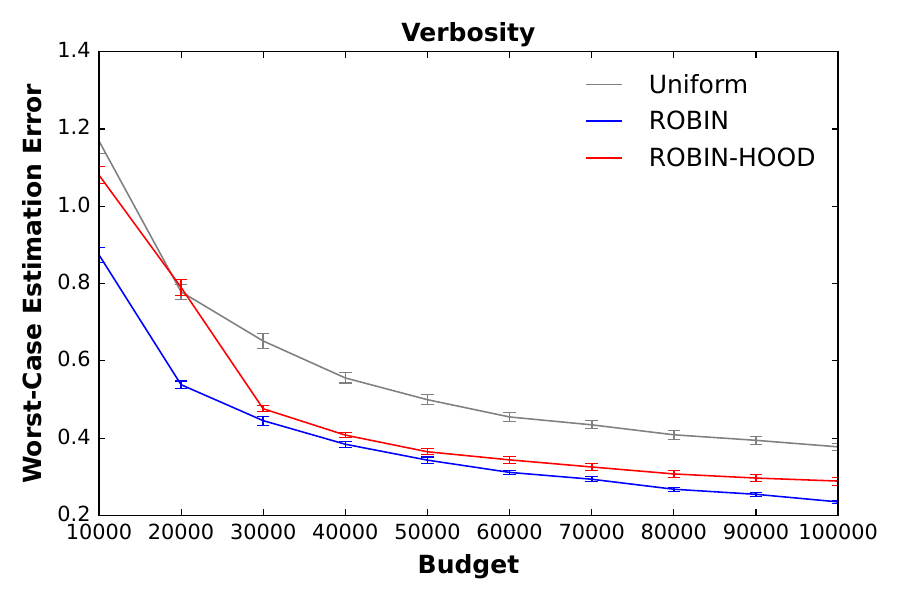}
	\end{minipage}
	
	\caption{y-axis: Maximum error (WCE). Top row: GPT-4.1 nano, $\delta$=0.007, warm-up period: 20164 samples. Bottom row: Llama 3.1 8B instruct, $\delta$=0.007, warm-up period: 20164 samples.}
	\label{fig:errors_full}
\end{figure*}

\subsection{Full Correlation with Human Ratings}

In \cref{fig:combined_correlations}, we extend \cref{fig:correlation_gpt}, \cref{fig:correlation_llama} to include all attribute datasets evaluated by GPT-nano and Llama 3.1. These graphs plot the correlation between human evaluations of prompt response pairs and judge-LLM using \cref{alg:algu} estimates of ratings of the same prompt response pairs. 

We compare three metrics, Pearson's coefficient, Spearman's coefficient, and Kendall's tau, over all attributes and both models. Here, we show how the judge-LLM's empirical estimates for prompt-response pairs highly correlate with human ratings for the same prompt-response pairs.

\begin{figure*}[t]
	\centering
	\begin{minipage}{0.24\textwidth}
		\centering
		\includegraphics[width=\textwidth]{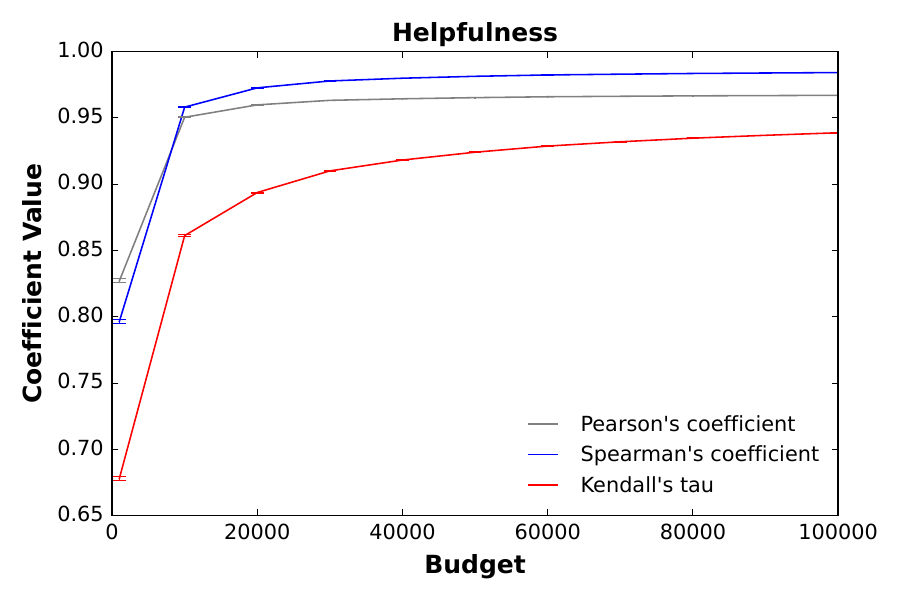}
	\end{minipage}
	\hfill
	\begin{minipage}{0.24\textwidth}
		\centering
		\includegraphics[width=\textwidth]{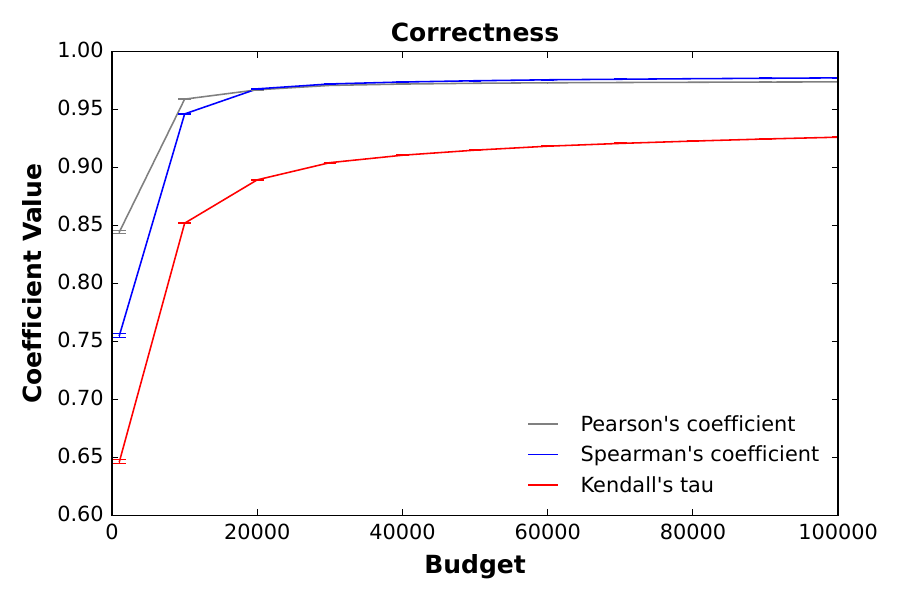}
	\end{minipage}
	\hfill
	\begin{minipage}{0.24\textwidth}
		\centering
		\includegraphics[width=\textwidth]{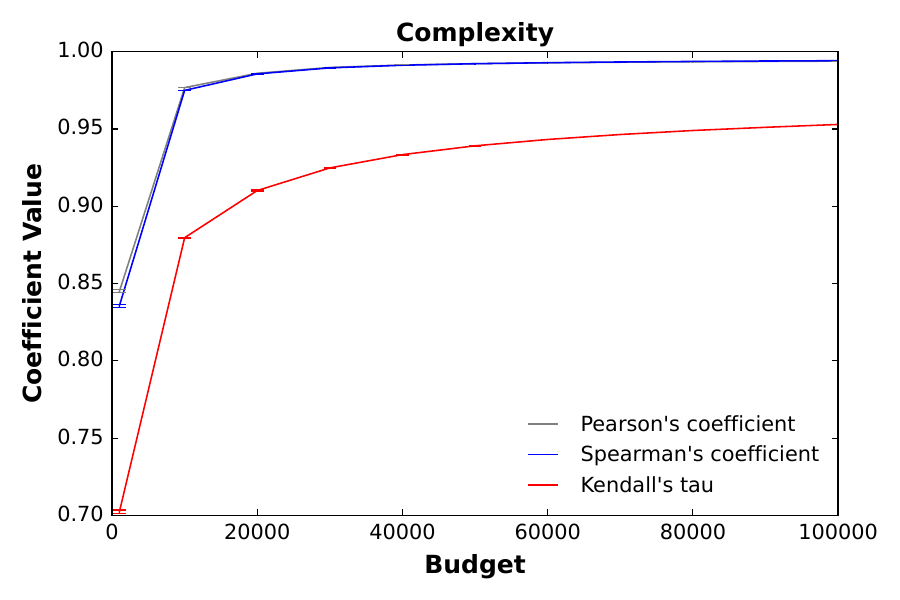}
	\end{minipage}
	\hfill
	\begin{minipage}{0.24\textwidth}
		\centering
		\includegraphics[width=\textwidth]{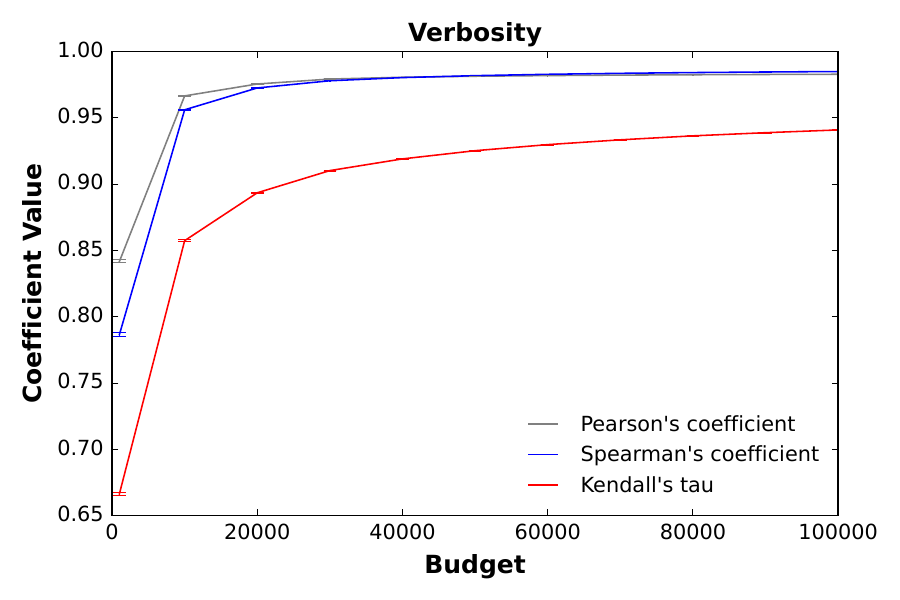}
	\end{minipage}
	
	
	\begin{minipage}{0.24\textwidth}
		\centering
		\includegraphics[width=\textwidth]{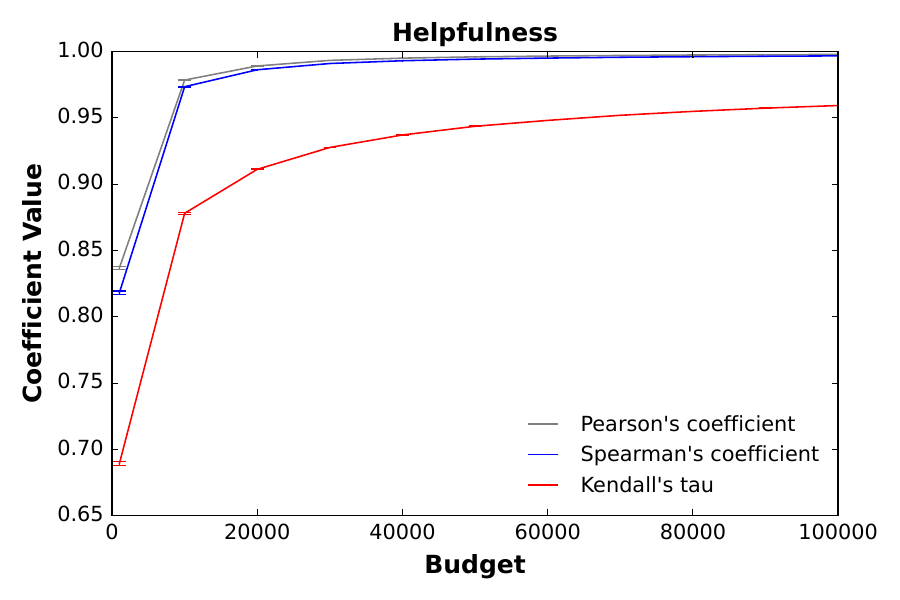}
	\end{minipage}
	\hfill
	\begin{minipage}{0.24\textwidth}
		\centering
		\includegraphics[width=\textwidth]{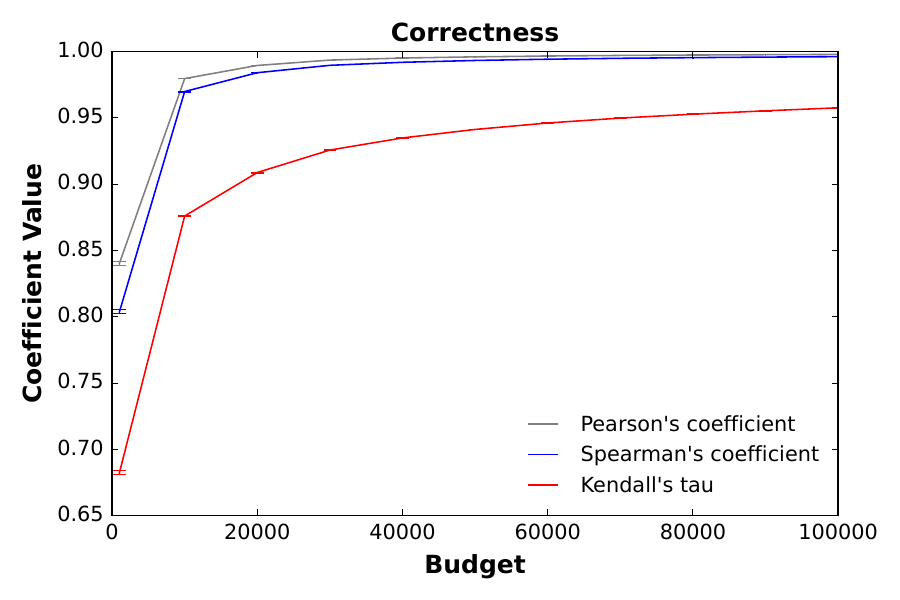}
	\end{minipage}
	\hfill
	\begin{minipage}{0.24\textwidth}
		\centering
		\includegraphics[width=\textwidth]{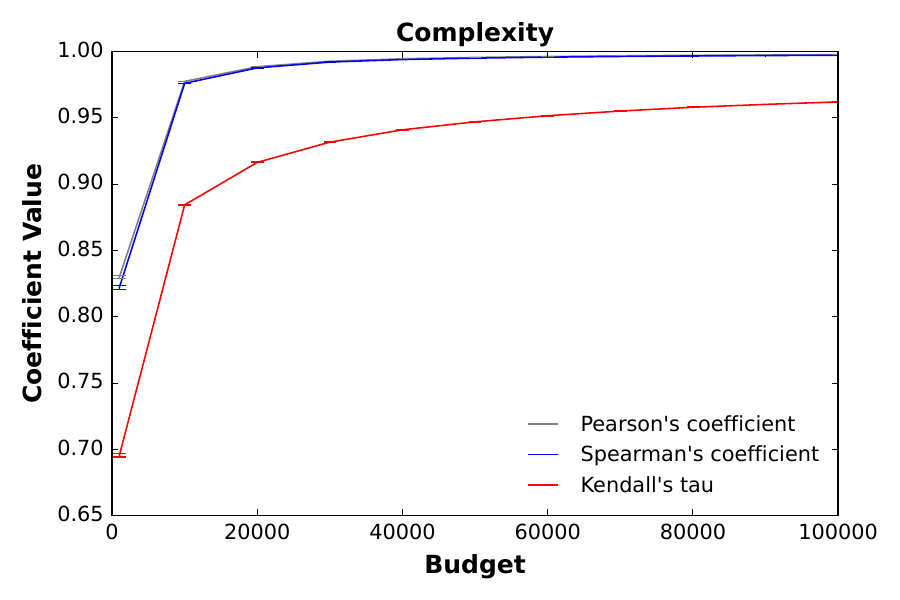}
	\end{minipage}
	\hfill
	\begin{minipage}{0.24\textwidth}
		\centering
		\includegraphics[width=\textwidth]{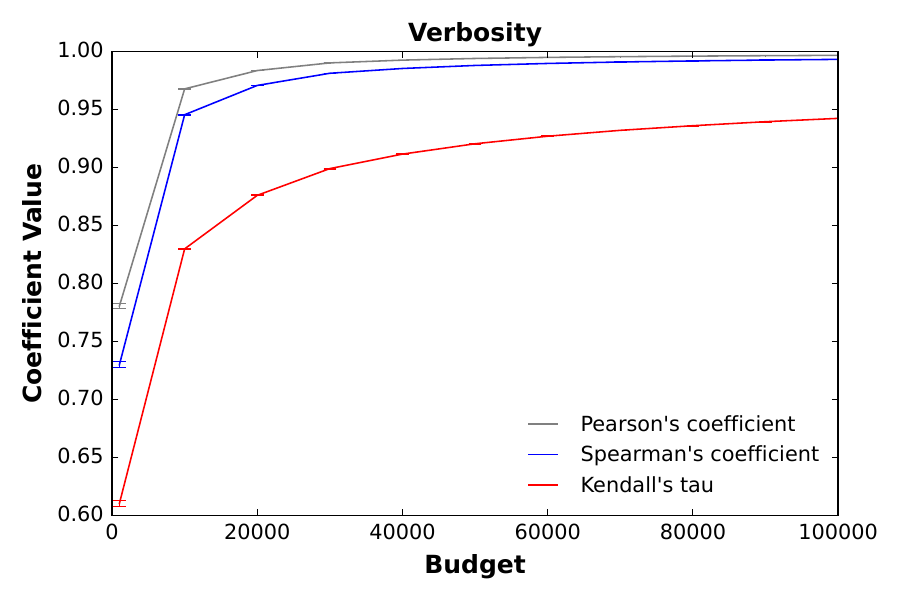}
	\end{minipage}
	
	\caption{y-axis: Correlations with Human Ratings. Top row: GPT-4.1 nano, $\delta$=0.007, warm-up period: 20164 samples. Bottom row: Llama 3.1 8B instruct, $\delta$=0.007, warm-up period: 20164 samples. Legend: (Red - Pearson's coefficient, Blue - Spearman's coefficient, Uniform - Kendall's Tau)}
	\label{fig:combined_correlations}
\end{figure*}

%% file: appendix_prompts.tex
\subsection{Prompts}
This section details the exact prompts used to generate our dataset described in \cref{sec:exp_setup}. 

\begin{tcolorbox}[colback=gray!5, colframe=black!50, title=Absolute LLM judge prompt, fonttitle=\bfseries, left=1pt, right=1pt, top=1pt, bottom=1pt, boxsep=4pt]
\scriptsize

\textbf{System Prompt:} You are a fair judge tasked with providing clear and objective feedback based on specific criteria. \\

\#\#\#Task description:
An instruction (might include an input inside it), a response to evaluate, and a score rubric representing an evaluation criterion are given.\newline
1. Write detailed feedback that assesses the quality of the response strictly based on the given score rubric, not evaluating in general.\newline
2. After writing the feedback, write a score that is an integer between 0 and 4. You should refer to the score rubric.\newline
3. The output should contain two lines in exactly this order:\newline
- Feedback: Detailed feedback that assesses the quality of the response\newline
- Rating: Score of the response\newline
4. Please do not generate any other opening, closing, and explanations.\\

\#\#\#Instruction to evaluate: 
\{instruction\} \\

\#\#\#Response to evaluate:
\{response\} \\

\#\#\#Score rubrics:
\{attributes\} \\

\#\#\#Output:
\end{tcolorbox}

The \{instruction\} and \{response\} are taken directly from prompt-response pairs in the Helpsteer2 dataset \citep{wang2024helpsteer2}. While the \{attributes\} are replaced with the corresponding attribute text from \cref{tab:attributes}. 

Using this method, we took 1000 prompt-response pairs from Helpsteer2 and created 30 judge-LLM evaluations using each Llama 3.1 8B instruct, and GPT-4.1 nano. Creating a total of 60K datapoints for each attribute. Thus, in total, we created and evaluated our algorithm using 240K samples of LLMs acting as judges.

\begin{table}[t!]
    \tiny
    \centering
    \begin{tabular}{lp{9.5cm}}
        \toprule
        Attribute & Attribute text\\
        \midrule
        Helpfulness & Helpfulness:\newline 
Score 0: The response is not useful or helpful at all. The response completely missed the essence of what the user wanted.\newline 
Score 1: The response is borderline unhelpful and mostly does not capture what the user was looking for, but it is still usable and helpful in a small way.\newline 
Score 2: The response is partially helpful but misses the overall goal of the user's query/input in some way. The response did not fully satisfy what the user was looking for.\newline 
Score 3: The response is mostly helpful and mainly aligned with what the user was looking for, but there is still some room for improvement.\newline 
Score 4: The response is extremely helpful and completely aligned with the spirit of what the prompt was asking for.\\
        \midrule
        Correctness & Correctness:\newline 
Score 0: The response is completely incorrect. All information provided is wrong, false or hallucinated. If the prompt asks the assistant to do a task, the task is not at all attempted, or the wrong task was attempted in the response. The response is completely irrelevant to the prompt.\newline 
Score 1: The response has some correct elements but is mostly wrong or incomplete. The response may contain multiple instances of hallucinations, false information, misleading information, or irrelevant information. If the prompt asks the assistant to do a task, the task was attempted with a small amount of success.\newline 
Score 2: The response contains a mix of correct and incorrect information. The response may miss some details, contain misleading information, or minor hallucinations, but is more or less aligned with what the prompt asks for. If the prompt asks the assistant to perform a task, the task is attempted with moderate success but still has clear room for improvement.\newline 
Score 3: The response is mostly accurate and correct with a small amount of missing information. It contains no misleading information or hallucinations. If the prompt asks the assistant to perform a task, the task is mostly successfully attempted.\newline 
Score 4: The response is completely correct and accurate to what is requested by the prompt with no necessary details missing and without false, misleading, or hallucinated information. If the prompt asks the assistant to do a task, the task is completely done and addressed in the response.\\
        \midrule
        Complexity & Complexity:\newline 
Score 0: (Basic) – The response uses very easy to understand language that is clear and completely interpretable by children, adults, and anyone with a functional command of the language.\newline 
Score 1: (Simple) – The response uses relatively straightforward language and wording, but some schooling through elementary or a middle school in the language might be required to understand the response.\newline 
Score 2: (Intermediate) – People who have completed up through a high school education will probably be able to understand the vocabulary and sentence structure used, but those at the basic level or children might struggle to understand the response.\newline 
Score 3: (Advanced) – The response uses a fairly sophisticated vocabulary and terminology. Someone majoring in this subject at a college or university could have written it and would understand the response. An average adult who does not work or study in this area could not have written the response.\newline 
Score 4: (Expert) – An expert in the field or area could have written the response. It uses specific and technically relevant vocabulary. Elevated language that someone at the simple or basic level may not understand at all. The professional language of a lawyer, scientist, engineer, or doctor falls into this category.\\
        \midrule
        Verbosity & Verbosity:\newline 
Score 0: (Succinct) – The response is short, to the point, and the most concise it can be. No additional information is provided outside of what is requested by the prompt (regardless of if the information or response itself is incorrect, hallucinated, or misleading. A response that gives an incorrect answer can still be succinct.).\newline 
Score 1: (Pretty Short) – The response is on the shorter side but could still have words, details, and/or text removed before it's at a bare minimum of what the response is trying to convey.\newline 
Score 2: (Average Length) – The response isn't especially long or short given what the prompt is asking of the model. The length is adequate for conveying a full response but isn't particularly wordy nor particularly concise.\newline 
Score 3: (Moderately Long) – The response is on the longer side but could still have more added to it before it is considered fully detailed or rambling.\newline 
Score 4: (Verbose) – The response is particularly lengthy, wordy, and/or extensive with extra details given what the prompt requested from the assistant model. The response can be verbose regardless of if the length is due to repetition and incoherency or if it is due to rich and insightful detail.\\
        \bottomrule \\
    \end{tabular}
    \caption{Attribute text for querying the Judge-LLM}
    \label{tab:attributes}
\end{table}